\documentclass{article}

\usepackage{microtype}
\usepackage{graphicx}
\usepackage{subfigure}
\usepackage{booktabs} 

\usepackage{hyperref}

\usepackage{amsmath}
\usepackage{amsfonts}
\usepackage{amsthm}
\usepackage{multirow}
\usepackage{mathtools}

\usepackage[accepted]{icml2019}

\icmltitlerunning{Beating Stochastic and Adversarial Semi-bandits Optimally and Simultaneously}

\newcommand{\E}{\mathbb{E}}
\newcommand{\X}{\mathcal{X}}
\DeclareMathOperator*{\argmin}{\arg\!\min}
\DeclareMathOperator*{\argmax}{\arg\!\max}
\DeclareMathOperator*{\conv}{\operatorname{Conv}}
\DeclareMathOperator*{\diag}{diag}
\DeclareMathOperator{\Reg}{\overline{Reg}}
\DeclareMathOperator{\Regpen}{Reg_{pen}}
\DeclareMathOperator{\Regstab}{Reg_{stab}}

\usepackage{color}
\usepackage[normalem]{ulem}

\newcommand{\commentout}[1]{}
\newcommand{\Alg}[1]{\textsc{#1}}
\newcommand{\scalar}[2]{\left\langle {#1},{#2} \right\rangle}
\newcommand{\LearningRate}{1/\sqrt{t}}
\newcommand{\Regularizer}{\sum_{i=1}^d -\sqrt{ x_i} + \gamma(1- x_i)\log(1- x_i)}
\newtheorem{lemma}{Lemma}
\newtheorem{theorem}{Theorem}

\begin{document}

\twocolumn[
\icmltitle{Beating Stochastic and Adversarial Semi-bandits \\ Optimally and Simultaneously}




\begin{icmlauthorlist}
\icmlauthor{Julian Zimmert}{ku}
\icmlauthor{Haipeng Luo}{LA}
\icmlauthor{Chen-Yu Wei}{LA}
\end{icmlauthorlist}

\icmlaffiliation{ku}{Department of Computer Science, University of Copenhagen, Copenhagen, Denmark}
\icmlaffiliation{LA}{Department of Computer Science, University of Southern California, United States}

\icmlcorrespondingauthor{Julian Zimmert}{zimmert@di.ku.dk}
\icmlcorrespondingauthor{Haipeng Luo}{haipengl@usc.edu}
\icmlcorrespondingauthor{Chen-Yu Wei}{chenyu.wei@usc.edu}

\icmlkeywords{Bandits, Online Learning, Best of Both Worlds, Online Mirror Descent, Tsallis Entropy, Multi-armed Bandits, Stochastic, Adversarial, I.I.D.}

\vskip 0.3in
]



\printAffiliationsAndNotice{} 

\begin{abstract}
We develop the first general semi-bandit algorithm that simultaneously achieves $\mathcal{O}(\log T)$ regret for stochastic environments
and $\mathcal{O}(\sqrt{T})$ regret for adversarial environments
without prior knowledge of the regime or the number of rounds $T$.
The leading problem-dependent constants of our bounds are not only optimal in a certain worst-case sense studied previously,
but also optimal for two concrete instances of semi-bandit problems.
Our algorithm and analysis extend the recent work of~\citet{zimmert2018optimal} for the special case of multi-armed bandits,
but importantly requires a novel hybrid regularizer designed specifically for semi-bandit.
Experimental results on synthetic data show that our algorithm indeed performs well over different environments.
Finally, we provide a preliminary extension of our results to the full bandit feedback.
\end{abstract}

\section{Introduction}
The multi-armed bandit is one of the most fundamental online learning problems with partial information feedback.
In this problem a learner repeatedly selects one of $d$ arms and observes its loss generated by the environment,
with the goal of minimizing her {\it regret},
the difference between her total loss and the loss of the best fixed arm in hindsight.
It is well known that in the stochastic environment where each arm's loss is drawn independently from a fixed distribution,
the minimax optimal regret is of order $\mathcal{O}(\log T)$ where $T$ is the number of rounds (dependence on all other parameters is omitted)~\citep{lai1985asymptotically},
while in the adversarial environment where each arm's loss can be completely arbitrary, the minimax optimal regret is of order $\mathcal{O}(\sqrt{T})$~\citep{auer2002nonstochastic}.

Several recent works~\citep{bubeck2012best, seldin2014one, auer2016algorithm, seldin2017improved, wei2018more, zimmert2018optimal} develop ``best-of-both-worlds'' results for multi-armed bandits
and propose adaptive algorithms that achieve $\mathcal{O}(\log T)$ regret in stochastic environments 
while simultaneously ensuring worst-case robustness, that is, $\mathcal{O}(\sqrt{T})$ regret even for adversarial environments.
Importantly, this is achieved without any prior knowledge of the nature of the environment.

In this work, we extend such best-of-both-worlds results to the combinatorial bandit problem,
a generalization of multi-armed bandits, where the learner has to pick a subset of arms (called a combinatorial action) at each time
(see Section~\ref{sec:setting} for formal definitions).
In particular, we consider the {\it semi-bandit} feedback, 
meaning that the learner observes the loss of each arm in the selected subset.
Our main contributions include the following:
\begin{enumerate}
\item  We propose a simple and general semi-bandit algorithm based on the {\it Follow-the-Regularized-Leader} (\Alg{Ftrl})
framework with a novel regularizer 
(Section~\ref{sec:alg}).

\item For any combinatorial action set, we prove that our algorithm achieves $\mathcal{O}(C_{sto}\log T)$ regret
for stochastic environments and $\mathcal{O}(C_{adv}\sqrt{T})$ regret for adversarial environments,
where $C_{sto}$ and $C_{adv}$ are problem-dependent factors (that do not depend on $T$)
and are optimal in some worst-case sense.
This is the first best-of-both-worlds result for combinatorial bandit to the best of our knowledge 
(Section~\ref{sec:general_results}).

\item For two common special cases of combinatorial action sets: 
the set of all subsets of arms and the set of all subsets with a fixed size $m$ (so called $m$-set),
we further derive refined bounds for the problem-dependent constants $C_{sto}$ and $C_{adv}$,
which are optimal for each of these special cases.
As a side result, our bounds imply that for the $m$-set with $m > d/2$,
semi-bandit feedback is no harder than full-information feedback in the adversarial case
(Sections~\ref{sec:hypercube} and~\ref{sec:m-set}).

\item We conduct experiments with synthetic data to show that our algorithm indeed adapts well to the nature of the environment.
Additionally, we present a simple intermediate setting where our algorithm outperforms all baselines
(Section~\ref{sec:experiments}).

\item We also provide a preliminary extension of our results to a special case of the more challenging bandit feedback (Section~\ref{sec:ext}).
\end{enumerate}

Our techniques are close to those of~\citep{zimmert2018optimal}:
we make use of the \Alg{Ftrl} algorithm, a well-known framework for adversarial environments,
and show that with a simple time-decaying learning rate schedule (that is, $1/\sqrt{t}$ for time $t$),
the regret admits a certain {\it self-bounding} property under the stochastic environment which eventually leads to logarithmic regret in this case.
Importantly, however, our results require the use of a novel {\it hybrid} regularizer, designed specifically for semi-bandit.
Roughly speaking, the idea is that for arms outside of the optimal subset, the problem of identifying their suboptimality is analogous to the multi-armed bandit problem,
and we apply the regularizer of~\citet{zimmert2018optimal} to these arms;
and on the other hand for arms in the optimal subset, the problem behaves like the full-information expert problem~\citep{freund1997decision},
and we thus apply the classical Shannon entropy as the regularizer to these arms.

\subsection{Related work}
\paragraph{Semi-bandits.} 
The combinatorial semi-bandit problem is a natural generalization of multi-armed bandits
and captures many real-life applications.
There are many algorithms for stochastic semi-bandits based on the well-known optimistic principle~\citep{gai2012combinatorial, chen2013combinatorial, kveton2015tight, combes2015combinatorial}.
Optimistic algorithms are provably not instance-optimal~\citep{lattimore2016end} and a
recent work developed a general instance-optimal algorithm for any structured stochastic bandits (including semi-bandit as a special case~\citep{combes2017minimal}).
Specifically, they obtain the optimal regret $\mathcal{O}(C\log T)$ where $C$ is an instance-dependent term expressed as the solution of a certain optimization problem.
The constant $C_{sto}$ in our stochastic bound $\mathcal{O}(C_{sto}\log T)$ is also expressed as an optimization problem (see Theorem~\ref{th:arbitrary}),
but it is not clear how it compares to the instance-optimal constant $C$ in general, except for the two special cases we discuss in Section~\ref{sec:main_results}.
Two advantages of our algorithm compared to prior work are: 
a) our stochastic assumption is weaker than others (see Section~\ref{sec:setting}) and
b) our algorithm ensures worst-case robustness even when the stochastic assumption does not hold.

Algorithms with $\mathcal{O}(\sqrt{T})$ regret for the adversarial semi-bandit setting are also well-studied~\citep{audibert2013regret, neu2013efficient, combes2015combinatorial, neu2015first, wei2018more}.
These algorithms are either based on Follow-the-Regularized-Leader (equivalently Online Mirror Descent) 
or Follow-the-Perturbed-Leader, both of which are standard frameworks for designing adversarial online learning algorithms
(see~\citet{hazan2016introduction} for an introduction).
It is easy to show that even if the environment is stochastic, the regret of these algorithms is still $\Theta(\sqrt{T})$,
indicating the lack of adaptivity.
Moreover, even for the adversarial case the leading constant in previous bounds is only worst-case optimal but not instance-optimal.
In contrast, our adversarial regret bound $\mathcal{O}(C_{adv}\sqrt{T})$ is instance-dependent through the term $C_{adv}$, again expressed as the solution of a certain optimization problem (see Theorem~\ref{th:arbitrary}).
To the best of our knowledge, there is no known general instance-dependent lower bound for this term,
but again we show the optimality of our bound in two special cases in Section~\ref{sec:main_results}.

\paragraph{Best-of-both-worlds.} 
Algorithms that are optimal for both stochastic and adversarial environments were studied for multi-armed bandits~\cite{bubeck2012best, seldin2014one, auer2016algorithm, seldin2017improved, wei2018more, zimmert2018optimal},
and also for the easier full-information (the expert problem)~\citep{gaillard2014second, luo2015achieving, koolen2016combining} and intermediate version~\citep{thune2018adaptation}.
Notably, among these works the recent two~\citep{wei2018more, zimmert2018optimal} discovered that sophisticated hypothesis testing or gap estimations used in earlier works are in fact not needed for such adaptivity.
Instead, their algorithms are based on the \Alg{Ftrl} framework with special regularizers.
As mentioned, our work also follows this route by designing a new regularizer for the more general semi-bandit setting.


\paragraph{Hybrid regularizers.}
The idea of using hybrid regularizers for \Alg{Ftrl} was first proposed by~\citet{bubeck2018sparsity} for sparse bandit and bandit with a specific form of adaptive regret bound,
and also recently used by~\citet{luo2018efficient} for the online portfolio selection problem.
The form of the hybrid regularizers and the way they are used in the analysis, however, are different both among these two prior works and with ours.

\section{Problem Setting and Algorithm}
\label{sec:setting}
The semi-bandit problem is a sequential game between a learner and an environment with $d$ fixed arms.
We call a subset of arms a combinatorial action,\footnote{In some works a combinatorial action is also referred to as ``an arm'',
but here we exclusively use the term ``arm'' for one of the $d$ elements and ``combinatorial action'' for a subset of these elements.}
and the learner is given a fixed set of combinatorial actions $\X\subset \{0,1\}^d$.
At any time $t=1,2,\dots$, the learner chooses an action $X_t\in\X$ and at the same time the environment chooses a loss vector $\ell_t\in [-1,1]^d$. 
The learner suffers the loss $\scalar{X_t}{\ell_t}$ and receives the feedback $o_t = X_t\circ\ell_t$, where $\circ$ stands for the element-wise multiplication.
In other words, the learner only observes the loss of each arm in the selected subset (the so-called semi-bandit feedback).

The environment can be either {\it stochastic} or {\it adversarial}.
In the stochastic case, we adopt and extend the broader ``stochastically constrained adversarial setting''~\citep{wei2018more, zimmert2018optimal} and assume that
there is a fixed action
$x^* \in \X$ such that for any $x \in\X\backslash\{x^*\}$ there exists a constant $\Delta_x > 0$, such that $\E[\scalar{x-x^*}{\ell_t}] \geq \Delta_x$ for all $t$. 
Note that this clearly subsumes the traditional stochastic setting where $\ell_1, \ldots, \ell_T$ are i.i.d. samples from a fixed unknown distribution,
and is much more general since neither independence nor identical distributions are  required.
In the adversarial case, on the other hand,
$\ell_t$ is chosen in an arbitrary way based on the history $\ell_1, X_1, \ldots, \ell_{t-1}, X_{t-1}$ and possibly an internal randomization by the environment.

The performance of a learner is measured by {\it pseudo-regret}:
\begin{align*}
\Reg_T :=  \E\left[\sum_{t=1}^T\scalar{X_t-x^*}{\ell_t}\right],
\end{align*}
where $x^* = \argmin_{x\in\X}\E\left[\sum_{t=1}^T\scalar{x}{\ell_t}\right]$ is the best action in hindsight and the expectation is with respect to the randomness of both the learner and the environment.
Note that in the stochastic case we are overloading the notation $x^*$ since clearly they are the same action.

It is well known that in terms of the dependence on $T$,
the optimal regret is $\Theta(\log T)$ in the stochastic case and $\Theta(\sqrt T)$ in the adversarial case (see, for example,~\citet{audibert2013regret,combes2017minimal}).

\paragraph{Notations.} We denote by $\E_t[\cdot]$ the conditional expectation $\E[\cdot | \mathcal{F}_{t-1}]$ where $\mathcal{F}_t$ is the filtration $\sigma(X_1, o_1, \ldots, X_t, o_t)$. 
We also use a shorthand $\mathbb{I}_t(i)$ for the indicator function $\mathbb{I}\{X_{ti}=1\}$ ($X_{ti}$ is the $i$-th component of the vector $X_t\in\X\subset \{0,1\}^d$) and write the characteristic function of a set $A$ as $\mathcal{I}_A(x)$
which is $0$ if $x\in A$ and $+\infty$ otherwise. 
We denote the $d$-dimensional vector with all $1$s as $\mathbf{1}_d$.

\subsection{Our algorithm}\label{sec:alg}
Our algorithm is based on the general \Alg{Ftrl} framework.\footnote{For linear objectives and Legendre regularizers, \Alg{Ftrl} is equivalent to Online Mirror Descent as defined in \citep{orabona2015generalized}. The same framework is also known under the names \Alg{Omd}, \Alg{Osmd}, or \Alg{Inf}.} 
In this framework, each time the algorithm computes the regularized leader $x_t = \argmin_{ x \in \conv(\X)} \big\langle x, \hat{L}_{t-1}\big\rangle + \eta_t^{-1}\Psi(x)$,
where $\conv(\X)$ is the convex hull of $\X$, 
$\hat{L}_{t-1} = \sum_{s=1}^{t-1} \hat{\ell}_s$ is the cumulative estimated loss,
$\eta_t > 0$ is a learning rate,
and $\Psi(x): \conv(\mathcal{X})\rightarrow \mathbb{R}\cup\{+\infty\}$ is a regularizer.
Then the algorithm samples $X_t \sim P(x_t)$ for a sampling rule $P$ that provides a distribution over $\mathcal{X}$ satisfying $\E_{X\sim P(x)}[X]= x$.
As long as $\conv(\X)$ can be described by a polynomial number of constraints,
one can always find an efficient sampling rule $P$ (see concrete examples in Section~\ref{sec:main_results}).
Finally, the algorithm constructs a loss estimator $\hat{\ell}_t$ based on the observed information and proceeds to the next round.


The novelty of our algorithm lies in the use of the hybrid regularizer 
\begin{equation}\label{eq:reg}
\Psi( x) =\Regularizer
\end{equation} 
with a parameter $0<\gamma\leq 1$ to be chosen later based on the action set $\X$ (in most cases we use $\gamma=1$).
This is a combination of the Tsallis entropy (with power $1/2$) $\sum_i- \sqrt{x_i}$,
and the Shannon entropy $\sum_i (1- x_i)\log(1- x_i)$ on the {\it complement} of $x$.
The $\sum_i -\sqrt{x_i}$ regularizer was first implicitly introduced by \citet{audibert2009minimax},
and later discovered as a member of the Tsallis entropy regularizers by~\citet{abernethy2015fighting}. It was also recently shown to be optimal for both stochastic and adversarial multi-armed bandits~\citep{zimmert2018optimal}.

In addition, similar to~\citet{zimmert2018optimal}, our algorithm uses a very simple time-decaying learning rate schedule $\eta_t = \LearningRate$.
The loss estimators $\hat{\ell}_t$ are defined as $\hat\ell_{ti} = \frac{(o_{ti}+1)\mathbb{I}_t(i)}{ x_{ti}}-1$ for all $i$.
It is clear the estimators are unbiased, $\E_t[\hat{\ell}_t] = \ell_t$, just as common importance weighted estimators.
The shift by $1$ is used to ensure that the range of the loss estimates is bounded from one side, $\hat{\ell}_{t,i} \geq -1$.
See Algorithm~\ref{alg:main_alg} for a complete pseudocode.

\begin{algorithm}[tb]
\caption{\Alg{Ftrl} with hybrid regularizer for semi-bandits}
\label{alg:main_alg}
\begin{algorithmic}
   \STATE {\bfseries Input:} $0<\gamma \leq 1$, sampling scheme $P$
   \STATE {\bfseries Initialize:} $\hat{L}_0 = (0, \ldots, 0), \eta_t = 1/\sqrt{t}$
   \FOR{$t=1, 2, \dots$}
   \STATE compute \[ x_t = \argmin\limits_{ x \in \conv(\mathcal{X})} \big\langle x,\hat{L}_{t-1}\big\rangle + \eta_t^{-1}\Psi( x)\] where $\Psi(\cdot)$ is defined in Eq.~\eqref{eq:reg}\;
   \STATE sample $X_t \sim P(x_t)$\;
   \STATE observe $o_t = X_t \circ \ell_t$\;
   \STATE construct estimator $\hat\ell_t,\; \forall i:\; \hat\ell_{ti} = \frac{(o_{ti}+1)\mathbb{I}_t(i)}{ x_{ti}}-1 $ \;
   \STATE update $\hat L_t = \hat L_{t-1}+\hat\ell_t$\;
   \ENDFOR
\end{algorithmic}
\end{algorithm}

\paragraph{Intuition behind the new regularizer.}
It is known that the classical Shannon entropy regularizer~\citep{freund1997decision} is optimal for both adversarial and stochastic environments in the full-information setting.
In fact, the Shannon entropy on the {\it complement} of $x$ is also optimal for full-information. 
This can be verified by considering the complementary problem: the problem with action set $\mathbf{1}_d-\X$ and reversed losses $-\ell_t$.
Both problems describe the exact same game with the same information,
and using Shannon entropy in the complementary problem is the same as using it on the complement of $x$ in the original problem.

The intuition behind combining Tsallis and Shannon entropy is that when $x_i$ is close to $0$, the learner is starved of information and has to act similarly to a regular bandit problem.
The magnitude of the gradient and its slope in that regime are dominated by the Tsallis entropy, which again is known to be optimal for bandits.

On the other hand, when $x_i$ is close to $1$, the game resembles a full-information game, 
and Shannon entropy on the complement becomes the dominating part of the regularizer in that regime.
Effectively, this allows us to regularize arms in the optimal combinatorial set differently than arms outside the optimal set,
without the need to know which arms are in the optimal set.



\section{Main Results}
\label{sec:main_results}
In this section we present general regret guarantees for our algorithm, 
followed by concrete instantiations in two special cases.

\subsection{Arbitrary action set}
\label{sec:general_results}

To state the general regret bound for our algorithm for any arbitrary action set $\X$,
we define the following two functions:
\begin{align*}
&f(x) = \sum_{i:x^*_i=0}\sqrt{x_i}\\
&g(x) = \sum_{i:x^*_i=1}(\gamma^{-1}-\gamma\log(1-x_i))(1-x_i)
\end{align*}
and the instantaneous regret function $r:[0,\infty)^{|\X|}\rightarrow \mathbb{R}$ as
\begin{align*}
    r(\alpha) = \sum_{x\in\X\setminus\{x^*\}} \alpha_x\Delta_x
\end{align*}
(recall the definition of $x^*$ and $\Delta_x$ from Section~\ref{sec:setting}).
We also define $\overline\alpha = \sum_{x\in\X}\alpha_xx$ for any $\alpha \in [0,\infty)^{|\X|}$,
and let $\Delta({\X})$ denote the simplex of distributions over $\X$.

\begin{theorem}
\label{th:arbitrary}
For any $\gamma\leq 1$ the pseudo regret of Algorithm~\ref{alg:main_alg} is upper bounded by
\begin{align*}
\Reg_T \leq \mathcal{O}\left(C_{sto}\log{T}\right) + \mathcal{O}\left(C_{add}\right)
\end{align*}
in the stochastic case and
\begin{align*}
\Reg_T \leq \mathcal{O}\left(C_{adv}\sqrt{T}\right)\end{align*}
in the adversarial case,
where $C_{sto}$, $C_{add}$ and $C_{adv}$ are defined as
\begin{align*}
C_{sto} &:= \max_{\alpha \in [0,\infty)^{|\mathcal{X}|}} f(\overline\alpha)-r(\alpha), \\
C_{add} &:= \sum_{t=1}^{\infty} \max_{\alpha \in \Delta(\X)} \left( \frac{100}{\sqrt{t}}g(\overline \alpha)-r(\alpha) \right), \\
C_{adv} &:= \max_{x\in\conv(\X)} f(x)+g(x).
\end{align*}
Moreover, it always holds that 
$C_{sto}= \mathcal{O}\left(\frac{md}{\Delta_{\min}}\right)$,
$C_{add}= \mathcal{O}\left(\frac{m^2}{\gamma^2\Delta_{\min}}\right)$,
and $C_{adv}= \mathcal{O}\left(\frac{1}{\gamma}\sqrt{md}\right)$,
where $m=\max_{x\in \X}||x||_1$ 
and $\Delta_{\min}=\min_{x\in\X\setminus \{x^*\}}\Delta_x$.
\end{theorem}


We defer the proof to Section~\ref{sec:proof}.
The dependence of our bounds on $T$ is optimal in both cases.
The leading problem-dependent constants $C_{sto}$ and $C_{adv}$ are expressed as solutions to optimization problems.
Recent works~\citep{combes2015combinatorial, lattimore2016end, combes2017minimal} also expressed the instance-optimal leading constant in the stochastic case in a similar way, but it is not clear how to compare the results.

The explicit upper bounds on these constants stated at the end of the theorem immediately imply that for $\gamma = 1$ our bounds are worst-case optimal according to~\citep{kveton2015tight} and~\citep{audibert2013regret}.
Here, worst-case optimality refers to the minimax regret over all environments with the same value $m$ of $\max_{x\in \X}||x||_1$
and also the same value $\Delta_{\min}$ of $\min_{x\in\X\setminus \{x^*\}}\Delta_x$ in the stochastic case.

However, for explicit instances, one can hope to achieve even better bounds.
By exploiting the structure of the problem and providing better bounds on the constants $C_{sto}$, $C_{add}$ and $C_{adv}$, we show in the next two sections that our algorithm is optimal in two special cases.
For better interpretability, in the stochastic case we consider the more traditional setting where $\ell_1, \ldots, \ell_T$ are i.i.d. samples from an unknown distribution $\mathcal{D}$.
It is clear that we can define $\Delta_x = \E_{\ell\sim\mathcal{D}}[\scalar{x-x^*}{\ell}]$ in this case.

%

\subsection{Special case: full combinatorial set}
\label{sec:hypercube} 
The simplest semi-bandit problem is when $\X = \{0,1\}^d$,
that is, the learner can pick any subset of arms.
In this case $\conv(\X) = [0,1]^d$ and a trivial sampling rule is $P(x) = \bigotimes_{i=1}^d \text{Ber}( x_i)$ where $\text{Ber}(\cdot)$ stands for Bernoulli distribution.

It is clear that in this case each dimension/arm can be treated completely independently.
Note, however, that the problem of each dimension is not exactly a two-armed bandit problem since the loss of ``not choosing the arm'' is known to be $0$,
and the problem is asymmetric between positive and negative losses. 
Specifically, we prove the following regret guarantee for our algorithm,
where in the stochastic case with a slight abuse of notation we define $\Delta_i = \E_{\ell\sim\mathcal{D}}\left[\ell_{i}\right]$.
\begin{theorem}
\label{th:full}
If $\X = \{0,1\}^d$, the pseudo-regret of Algorithm~\ref{alg:main_alg} with $\gamma=1$ is 
\begin{align*}
\Reg_T \leq \mathcal{O}\left(\sum_{\Delta_i > 0} \frac{\log(T)}{\Delta_i}\right)+\mathcal{O}\left(\sum_{\Delta_i < 0} \frac{1}{|\Delta_i|}\right)
\end{align*}
in the stochastic case and
\begin{align*}
\Reg_T \leq \mathcal{O}\left(d\sqrt{T}\right)
\end{align*}
in the adversarial case. Moreover, both bounds are optimal.
\end{theorem}


\begin{proof}
Note that in this case the algorithm is equivalent to the following:
for each coordinate, run a copy of Algorithm~\ref{alg:main_alg} for a one-dimensional problem with $\X=\{0,1\}$ as the action set.
We can thus apply Theorem~\ref{th:arbitrary} to such one-dimensional problems and finally sum up the regret along each coordinate. Below we focus on a fixed coordinate $i$.

In particular, in the stochastic case, if $\Delta_i > 0$, it implies $x^*_i = 0$ and thus
$g(\cdot) \equiv 0$ and $C_{add} = \sum_{t} \max_{\alpha\in[0,1]} -\alpha \Delta_i = 0$. For $C_{sto}$ we apply the general bound from Theorem~\ref{th:arbitrary}
and obtain $C_{sto} =  \mathcal{O}\left(1/\Delta_i\right)$ (since $m=d=1$ and $\Delta_{\min} = \Delta_i$).
This gives the bound $\mathcal{O}\left(\frac{\log(T)}{\Delta_i}\right)$ for $\Delta_i > 0$.

On the other hand if $\Delta_i < 0$ then $x^*_i = 1$ and $f(\cdot) \equiv 0$, 
so $C_{sto} = \max_{\alpha \geq 0} \alpha\Delta_i = 0$.
For $C_{add}$ we apply the general bound from Theorem~\ref{th:arbitrary}
and obtain $C_{add} =  \mathcal{O}\left(1/\Delta_i\right)$ (since $m=\gamma=1$ and $\Delta_{\min} = \Delta_i$).
This gives the bound $\mathcal{O}\left(\frac{1}{\Delta_i}\right)$ for $\Delta_i < 0$.

In the adversarial case, we apply the general bound of Theorem~\ref{th:arbitrary} and obtain $C_{adv} =\mathcal{O}(1)$. This finishes the proof for the regret upper bounds.
The optimality of the adversarial bound is trivial since it matches the full-information lower bound.
Obtaining a matching lower bound in the stochastic regime is a simple adaptation of the regular two-armed bandit lower bound.
We believe this result is well known, but provide a proof in the appendix in absence of a reference.
\end{proof}

\subsection{Special case: $m$-set}
\label{sec:m-set}
Another common instance of semi-bandit is when the learner can only select subsets of a fixed size.
Specifically, let $m \in \{1, \ldots, d-1\}$ be a fixed parameter and define the $m$-set as
\begin{equation}\label{eq:m-set}
\X =\left\{x\in\{0,1\}^d \;\;\middle|\;\; \sum_{i=1}^dx_i = m\right\}.
\end{equation}
Note that we are overloading the notation $m = \max_{x\in \X}||x||_1$ since clearly they are the same in this case.
It is well-known that the convex hull of $m$-set is $\conv(\X) = \left\{x\in[0,1]^d \;\;|\;\; \sum_{i=1}^dx_i = m\right\}$, 
and in the appendix we provide a simple sampling rule $P$ with
$\mathcal{O}(d\log(d))$ time complexity. 
This improves over previous work that requires $\mathcal{O}(d^2)$ time complexity \citep{warmuth2008randomized,suehiro2012online}.

In the stochastic case, we assume without loss of generality that the expected losses of arms are increasing in $i$.
Overloading the notation again we define the stochastic gaps as $\Delta_i = \E_{\ell\sim\mathcal{D}}\left[\ell_i -\ell_m\right]$ for all $i$. 
Note that the uniqueness of $x^*$ also implies $\Delta_i \neq 0$ for all $i > m$.
The next theorem shows that our algorithm is optimal for both environments.
As a side result, we also show that when $m > d/2$, 
semi-bandit feedback is no harder than full-information feedback in the adversarial case.
To the best of our knowledge, this was previously unknown.

\begin{theorem}
\label{th:m-set}
If $\X$ is the $m$-set defined by Eq.~\eqref{eq:m-set},
then the pseudo-regret of Algorithm~\ref{alg:main_alg} with 
\[
\gamma=\begin{cases}
1 &\mbox{ if } m\leq d/2\\
\min\{1,1/\sqrt{\log(d/(d-m))}\} &\mbox{ otherwise, } 
\end{cases}
\] 
satisfies
\begin{align*}
\Reg_T \leq \mathcal{O}\left(\sum_{i =m+1}^d \frac{\log(T)}{\Delta_i}\right) + \mathcal{O}\left(\sum_{i=m+1}^d\frac{(\log d)^2}{\Delta_{i}}\right)
\end{align*}
in the stochastic case and
\begin{align*}
\Reg_T \leq \begin{cases}
\mathcal{O}\left(\sqrt{mdT}\right) &\mbox{ if }m\leq d/2\\
\mathcal{O}\left((d-m)\sqrt{\log(\frac{d}{d-m})T} \right) &\mbox{ otherwise } \\
\end{cases}
\end{align*}
in the adversarial case. Moreover, both bounds are optimal.
\end{theorem}
\begin{proof}[Proof sketch.]
We provide a proof sketch here and defer some details to Appendix~\ref{app:theorems}.

$\mathbf{C_{adv}}:$ The optimization problem is concave in $x$ and symmetric for all $i$ with the same value of $x^*_i$.
Therefore the optimal solution takes the form 
\begin{align*}
\left(\argmax_{x\in\conv(\X)}f(x) + g(x)\right)_i =  \begin{cases}\lambda &\mbox{ if }x^*_i = 0\\1-\frac{d-m}{m}\lambda &\mbox{ if }x^*_i = 1\end{cases}
\end{align*}
for some $\lambda\in [0,\min\{1,\frac{m}{d-m}\}]$.
In Appendix~\ref{app:theorems} we show that the function is increasing in $\lambda$,
and that inserting $\lambda=\min\{1,\frac{m}{d-m}\}$ leads to the stated adversarial bound.

$\mathbf{C_{sto}}:$ With the definitions of the gaps, we can express $\Delta_x = \sum_{i:x_i \neq x^*_i} |\Delta_i|$, which is lower bounded by $\sum_{i:x^*_i = 0, x_i=1} \Delta_i = \sum_{i:x^*_i = 0} \Delta_i x_i$. 
So the immediate regret function $r(\alpha)$ can be bounded as 
\begin{align*}
r(\alpha) &= \sum_{x\neq x^*} \Delta_x\alpha_x 
\geq \sum_{x\neq x^*} \sum_{i: x^*_i = 0} \Delta_i \alpha_x x_i \\ 
&= \sum_{i:x^*_i = 0}\Delta_i \left(\sum_{x\neq x^*}   \alpha_x x_i \right)
   = \sum_{i:x^*_i=0} \Delta_i\overline\alpha_i.
\end{align*}
The optimization problem can now be bounded as
\begin{align*}
C_{sto} &= \max_{\alpha\in[0,\infty)^{|\mathcal{X}|}}\sum_{i:x^*_i=0}
\sqrt{\overline\alpha_i}-\sum_{x\neq x^*} \alpha_x\Delta_x\\
&\leq \max_{\overline\alpha\in[0,\infty)^d}\sum_{i:x^*_i=0}\left(\sqrt{\overline\alpha_i}-\Delta_i\overline\alpha_i\right) = \sum_{i:x^*_i = 0} \frac{1}{4\Delta_i},
\end{align*}
which is the same as $\sum_{i=m+1}^d \frac{1}{4\Delta_i}$.

$\mathbf{C_{add}}:$ 
We bound the function $g$ as follows:
\begin{align*}
g(\overline\alpha)&=\sum_{i:x^*_i=1} (\gamma^{-1}-\gamma\log(1-\overline\alpha_i))(1-\overline\alpha_i) \\
&\leq \left(\gamma^{-1}-\gamma\log\left(\sum_{i:x^*_i=1}\frac{1-\overline\alpha_i}{m}\right)\right)\sum_{i:x^*_i=1}(1-\overline\alpha_i)\\
&=\left(\gamma^{-1}-\gamma\log\left(\sum_{i:x^*_i=0}\frac{\overline\alpha_i}{m}\right)\right)\sum_{i:x^*_i=0}\overline\alpha_i\\
&\leq\sum_{i:x^*_i=0}\left(\gamma^{-1}-\gamma\log\left(\frac{\overline\alpha_i}{m}\right)\right)\overline\alpha_i
\end{align*}
where the first inequality is by the concavity of $g$;
the second equality is by the fact $\sum_{i:x^*_i=1} 1-\overline\alpha_i = \sum_{i:x^*_i=0} \overline\alpha_i$
since $\overline\alpha$ is in the convex hull of $m$-set.

Recall the lower bound $r(\alpha)\geq \sum_{i:x^*_i=0} \Delta_i\overline\alpha_i$ as derived previously. We can thus bound $C_{add}$ as
\begin{align*}
&\sum_{i: x_i^*=0}\sum_{t=1}^{\infty} \max_{A \in[0,1]} \frac{100}{\sqrt{t}}\left(\gamma^{-1}-\gamma\log\left(\frac{A}{m}\right) \right)A -\Delta_{i}A
\end{align*}
Solving the one-dimensional optimization problems above independently for each $i$ (see Appendix~\ref{app:theorems}) proves $C_{add}\leq\mathcal{O}\left(\sum_{i:x^*_i=0}\frac{(\log d)^2}{\Delta_{i}}\right)$.

{\bf Optimality}:
The optimality for the stochastic case is implied by~\citep{anantharam1987asymptotically, combes2017minimal}.
For the adversarial case, only a matching lower bound $\Omega(\sqrt{mdT})$ for $m\leq d/2$ is known (Theorem~2 of~\citep{lattimore2018toprank}).
We close this gap by making a simple observation that when $m > d/2$, our bound in fact matches the lower bound of the same problem with full-information feedback.
This clearly implies the optimality of our bound since semi-bandit feedback is harder.

Indeed, \citet{koolen2010hedging} prove the lower bound $\Omega(m\sqrt{T\log(d/m)})$ for full-information $m$-set when $m \leq d/2$.
When $m > d/2$, one can simply work on the complementary problem with 
action set $\mathbf{1}_d-\X$ and reversed losses.
This is exactly a $(d-m)$-set problem and thus a lower bound
$\Omega((d-m)\sqrt{T\log(d/(d-m))})$ applies.
This exactly matches our upper bound.
\end{proof}

\section{Empirical Comparisons}
\label{sec:experiments}
We compare our novel algorithm with four baselines from the literature.
For stochastic algorithms, we choose \Alg{CombUCB}~\citep{kveton2015tight} and \Alg{Thompson Sampling}~\citep{gopalan2014thompson}; 
for adversarial algorithms, we choose \Alg{Exp2}~\citep{audibert2013regret} and \Alg{LogBarrier}~\citep{wei2018more}, which are respectively \Alg{Ftrl} with generalized Shannon entropy and log-barrier regularizer.
For each adversarial algorithm, we tune the time-independent part of the learning rate by choosing from the grid of $\{2^i| i\in\{-5,-4,\dots,5\}\}$,
and the optimal value happens to be identical for both adversarial and stochastic environment in our experiments.
Specifically the final learning rates $\eta_t$ for our algorithm,  \Alg{Exp2} and \Alg{LogBarrier} are respectively $1/\sqrt{t}$, $1/(4\sqrt{t})$ and $4\sqrt{\log(t)/t}$.

We test the algorithms on concrete instances of the $m$-set problem
with parameters: $d=10$, $m=5$, $T = 10^7$.
Below, we specify the mean of each arm's loss at each time. 
With mean $\mu_{ti}$ the actual loss of arm $i$ at time $t$ will be $-1$ with probability $(1-\mu_{ti})/2$ and $+1$ with probability $(1+\mu_{ti})/2$, independent of everything else.
We create the following two environments:
\paragraph{Stochastic environment.}
In this case the losses are drawn from a fixed distribution with
$\mu_{ti}= -\Delta$ if $i\leq 5$ and $\mu_{ti}=\Delta$ otherwise,
where $\Delta = 1/8$.
\paragraph{``Adversarial'' environment.}
Since it is difficult to create truly adversarial data,
here we in fact use a stochastically constrained adversarial setting defined in Section~\ref{sec:setting}.
The construction is similar to that of \citet{zimmert2018optimal}.
Specifically, the time is split into phases 
\begin{align*}
\underbrace{1,\dots,t_1}_{T_1},\underbrace{t_1+1,\dots,t_2}_{T_2},\dots,\underbrace{t_{n-1},\dots,T}_{T_n}.
\end{align*}
The length of phase $s$ is $T_s = 1.6^s$,
and the means of the losses are set to 
$$\mu_{ti}=\begin{cases}-\Delta/2 \pm (1-\Delta/2) &\mbox{ if }i \leq 5,\\+\Delta/2 \pm (1-\Delta/2) &\mbox{ otherwise, }  \end{cases},$$
where $\pm$ represents $+$ if $t$ belongs to an odd phase and $-$ otherwise.
%
%
This model is not only a nice toy example, but could also be justified by real world applications.
For example, in a network routing problem, an adversary might periodically attack the network, making the delay of every edge increase by roughly the same amount.


We measure the performance of the algorithms by the average pseudo-regret over at least 20 runs.
For \Alg{CombUCB} and \Alg{Thompson Sampling} in the adversarial environment, we increase the number of runs to 500 and 1000 respectively due to the high variance of the pseudo-regret.
Figure~\ref{fig:both} shows the average pseudo-regret of all algorithms at each time,
where plot (a) uses the stochastic data and plot (b) uses the adversarial data.
We use log-log scale after $10^4$ rounds.
Shaded areas in the plot show the confidence intervals.

\begin{figure}
    \centering
    \fontsize{6pt}{0.12pt}
    \def\svgwidth{\columnwidth}
    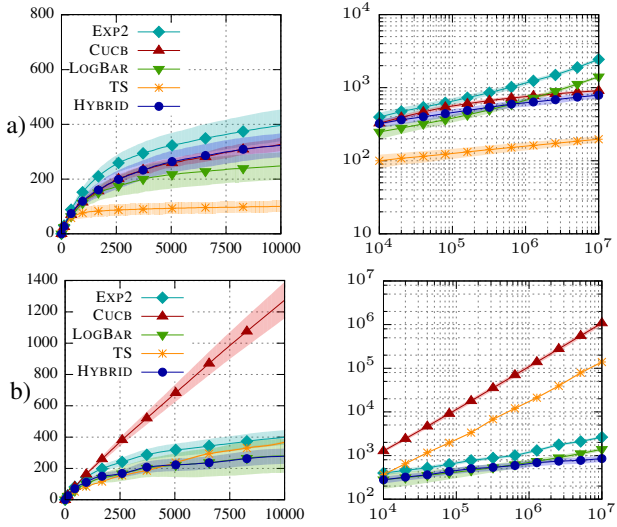
    \caption{Comparisons of our new algorithm (\Alg{Hybrid}) and several existing algorithms with $d=10, m=5$ and $T=10^7$ under a) stochastic and b) stochastically constrained adversarial setting. The left side is in linear scale and the right is in log-log scale.}
    \label{fig:both}
\end{figure}


The plots clearly confirm our theoretical results.
Our algorithm outperforms \Alg{Exp2} and \Alg{LogBarrier} (in the later stage) in both environments.
In the stochastic case our algorithm is competitive with \Alg{CombUCB}, 
while \Alg{Thompson Sampling} has the best performance (a well-known phenomenon).
However, these two stochastic algorithms clearly fail in the adversarial case
and exhibit nearly-linear regret.

\section{Proof of Theorem~\ref{th:arbitrary}}
\label{sec:proof}

We provide the key steps of the proof for our general result (Theorem~\ref{th:arbitrary}) in this section.
Define $\Psi_t(\cdot) = \eta_t^{-1}\Psi(\cdot)$ 
and potential function
$\Phi_t(\cdot) = \max_{ x \in \conv(\mathcal{X})} \left\langle x,\cdot\right\rangle - \Psi_t( x)$,
which is the convex conjugate of $\Psi_t+  \mathcal{I}_{\conv(\mathcal X)}$.

Following a standard analysis of FTRL, we decompose the regret
\begin{align}
    &\Reg_T = \underbrace{\E\left[\sum_{t=1}^T\scalar{X_t}{\ell_{t}} + \Phi_t(-\hat L_t) - \Phi_t(-\hat L_{t-1})\right]}_{\Regstab} \nonumber\\
    &+ \underbrace{\E\left[\sum_{t=1}^T -\Phi_t(-\hat L_t) + \Phi_t(-\hat L_{t-1}) - \scalar{x^*}{\ell_{t}}\right],}_{\Regpen}
\label{eq:regsplit}
\end{align}
into terms corresponding to the \textit{stability} and the \textit{regularization penalty} of the algorithm. 

We then further bound these two terms respectively in the following two lemmas using mostly standard \Alg{Ftrl} analysis (see Appendix~\ref{app:arbitrary} for the proofs).

\begin{lemma}
\label{lem:pen}
The regularization penalty is bounded as
\begin{align*}
\Regpen \leq &\sum_{t=1}^T \frac{3}{2\sqrt{t}}\Bigg(\sum_{i:x^*_i=0} \sqrt{\E[ x_{ti}]}\\
&-\sum_{i:x^*_i=1} \gamma(1-\E[ x_{ti}])\log(1-\E[ x_{ti}])\Bigg).
\end{align*}
\end{lemma}

\begin{lemma}
\label{lem:stab}
The stability term is bounded as
\begin{align*}
\Regstab \leq &\sum_{t=1}^T \frac{16\sqrt{2}}{\sqrt{t}}\Bigg(\sum_{i:x^*_i=0} \sqrt{\E[ x_{ti}]}\\
&+\sum_{i:x^*_i=1} \gamma^{-1}(1-\E[ x_{ti}])\Bigg)+c .  
\end{align*}
where $c = 58 m/\gamma^2$ (recall that $m =\max_{x\in \X}||x||_1$).
\end{lemma}


We now proceed to the proof of Theorem~\ref{th:arbitrary}.
\begin{proof}[Proof of Theorem~\ref{th:arbitrary}]
Using Lemma~\ref{lem:pen} and Lemma~\ref{lem:stab} in Eq.~\eqref{eq:regsplit} and the definition of functions $f$ and $g$, we can bound the regret by
\begin{align}
\Reg_T &\leq \sum_{t=1}^T \frac{25}{\sqrt{t}}\left(f(\E[x_{t}])+g(\E[x_{t}])\right)+c  \label{eq:reg_bound} \\
&\leq 50\sqrt{T} \max_{x\in\conv(\X)}\left(f(x)+g(x)\right)+c \notag\\
&= \mathcal{O}\left(C_{adv}\sqrt{T}\right), \notag
\end{align}
which concludes the adversarial case.

For the stochastic case we use a {\it self-bounding} technique similar to~\citet{wei2018more, zimmert2018optimal}.
First, by the definition of the function $r$ and the stochastic assumption we have
\begin{align*}
\Reg_T = \E\left[\sum_{t=1}^T\scalar{\E[x_t] - x^*}{\ell_t}\right] \geq \sum_{t=1}^T r(P(\E[x_t])).
\end{align*}
Together with Eq.~\eqref{eq:reg_bound} we have
\begin{align*}
\sum_{t=1}^T \frac{25}{\sqrt{t}}\left(f(\E[x_{t}])+g(\E[x_{t}])\right)+c - \sum_{i=1}^T r(P(\E[x_t])) \geq 0. 
\end{align*}

Combining the above with Eq.~\eqref{eq:reg_bound} again we bound $\Reg_T$ by
\begin{align*}
\sum_{t=1}^T\left( \frac{50}{\sqrt{t}}\left(f(\E[x_{t}])+g(\E[x_{t}])\right)- r(P(\E[x_t]))\right)+2c.
\end{align*}

We next decompose the summation above into two terms and upper bound them as $C_{sto}\log T$ and $C_{add}$ respectively:
\begin{align*}
&\sum_{t=1}^T\frac{50}{\sqrt{t}}f(\E[x_{t}])- \frac{1}{2}r(P(\E[x_t]))\\
&\leq \sum_{t=1}^T  \max_{\alpha\in\Delta(\X)} \frac{50}{\sqrt{t}}f(\overline{\alpha})-\frac{1}{2}r(\alpha)\\
&\leq\sum_{t=1}^T  \max_{\alpha\in[0,\infty)^{|\X|}} \frac{50}{\sqrt{t}}f\left(\frac{10^4}{t}\overline{\alpha}\right)-\frac{1}{2}r\left(\frac{10^4}{t}\alpha\right)\\
&\stackrel{(\star)}{=} \sum_{t=1}^T \frac{10^4}{2t}\max_{\alpha\in[0,\infty)^{|\X|}}f(\overline{\alpha})-r(\alpha)= \mathcal{O}\left(C_{sto}\log(T)\right) 
\end{align*}
where $(\star)$ follows since $r$ is linear and $f$ satisfies for any scalar $a\geq 0$: $f(ax)=\sqrt{a}f(x)$. 
On the other hand, 
\begin{align*}
&\sum_{t=1}^T\frac{50}{\sqrt{t}}g(\E[x_{t}])- \frac{1}{2}r(P(\E[x_t])) \\
&\leq \frac{1}{2}\sum_{t=1}^\infty  \max_{\alpha\in \Delta(\X)} \left(\frac{100}{\sqrt{t}}g(\overline{\alpha})- r(\alpha)\right)=\mathcal{O}(C_{add}),
\end{align*}
where the last inequality uses the fact: for all $t>0$, $\max_{\alpha\in \Delta(\X)}\left(\frac{100}{\sqrt{t}}g(\overline{\alpha})- r(\alpha)\right)\geq 0$. This is because a particular  $\alpha$ that puts all the weight on $x^*$ attains the value of $0$. 

The above finishes the proof of the general regret bounds.
Due to space limitations we defer the derivation of upper bounds on the constants $C_{sto}, C_{add}$ and $C_{adv}$ to Appendix~\ref{app:arbitrary}.
\end{proof}

\section{Extensions to Bandit Feedback}
\label{sec:ext}
The most natural extension of our work is to consider the full bandit feedback setting,
where each time after playing an action $X_t$ the learner only observes $\scalar{X_t}{\ell_t}$.
Again, both stochastic and adversarial versions of the problem are well-studied in the literature,
but there is no best-of-both-worlds result.
Here, we provide a preliminary result for the simplest case $\X = \{0, 1\}^d$.
Following convention for this setting we also restrict $\ell_t$ to be such that $\|\ell_t\|_1\leq 1$.
Similar to Section~\ref{sec:hypercube},
in the stochastic case we assume $\ell_t \sim \mathcal{D}$
and define $\Delta_i = \E_{\ell\sim\mathcal{D}}[\ell_i]$.

\begin{theorem}
For the full bandit feedback setting with $\X = \{0, 1\}^d$ and $\|\ell_t\|_1\leq 1$,
\Alg{Ftrl} with regularizer $\Psi(x) = \sum_{i=1}^d \sqrt{x_i}+\sqrt{1-x_i}$, learning rate $\eta_t=1/\sqrt{t}$ and loss estimators $\hat\ell_{ti} = \frac{\scalar{X_t}{\ell_t}X_{ti}}{x_{ti}} - \frac{\scalar{X_t}{\ell_t}(1-X_{ti})}{1-x_{ti}}$ ensures:
\begin{align*}
\Reg_T \leq \mathcal{O}\left(\sum_{i: \Delta_i \neq 0} \frac{\log(T)}{|\Delta_i|}\right)
\end{align*}
in the stochastic case and
\begin{align*}
\Reg_T \leq \mathcal{O}\left(d\sqrt{T}\right)
\end{align*}
in the adversarial case. Moreover, both bounds are optimal.
\end{theorem}
\begin{proof}[Proof sketch.]
In this case, the optimization of \Alg{Ftrl} decomposes over the coordinates and
it is clear that the stated algorithm is equivalent to the following:
for each coordinate $i$, apply the algorithm of~\citet{zimmert2018optimal} to a two-armed bandit problem where the loss of arm 1 at time $t$ is $\ell_{ti}+\sum_{j\neq i}X_{tj}\ell_{tj}$ and the loss of arm 2 is $\sum_{j\neq i}X_{tj}\ell_{tj}$.\footnote{%
The losses are well defined since they do not depend on $X_{ti}$.
}
In the stochastic case this exactly fits into the stochastically constrained adversarial setting of~\citet{zimmert2018optimal} with gap $|\Delta_i|$ and, therefore, applying their Theorem~2 and summing up the regret over each coordinate finishes the proof for the stated regret bounds.
The optimality of the stochastic bound follows from~\citet{combes2017minimal} and the optimality of the adversarial bound follows from~\citet{dani2008price}.
%
\end{proof}

For general action sets, however, the problem becomes significantly harder,
because all known adversarial algorithms, e.g.~\citet{cesa2012combinatorial}, require implicit or explicit exploration of order $1/\sqrt{T}$, which prohibits $\log(T)$ regret in the stochastic case.
We leave this as question for future work.

\section{Conclusions}
We provide the first best-of-both-worlds results for combinatorial bandits,
via an \Alg{Ftrl}-based algorithm with a novel hybrid regularizer.
Our bounds are worst-case optimal and also optimal in two particular instances of the problem.
Empirical evaluations also confirm our theory.

Other than the open problem under bandit feedback mentioned in Section~\ref{sec:ext},
another open question is whether our stochastic bound is instance-optimal as in~\citet{combes2017minimal}, and if not, whether there is a best-of-both-worlds algorithm that is instance-optimal in the stochastic case.
One can also ask the same question for the adversarial case, 
however, next to nothing is known regarding the instance-optimality of the adversarial case,
let alone best-of-both-worlds results.


\paragraph{Acknowledgments}
HL and CYW are supported by NSF Grant \#1755781.
We thank Yevgeny Seldin for valuable feedback and discussions,
and Shinji Ito for pointing us to missing references and unclarities in Section~\ref{sec:ext}.

\bibliography{bibliography,mybib}
\bibliographystyle{icml2019}

\onecolumn
\appendix

\section{Omitted details for the Proof of Theorem~\ref{th:arbitrary}}
\label{app:arbitrary}

In this section we provide omitted details for the proof of Theorem~\ref{th:arbitrary}.
We first prove Lemmas~\ref{lem:pen} and~\ref{lem:stab},
then continue on Section~\ref{sec:proof} and prove the upper bounds for $C_{sto}$, $C_{add}$ and $C_{adv}$.

\subsection{Regularization penalty}
In order to bound the regularization penalty, we make use of the following standard result for \Alg{Ftrl}.
\begin{lemma}
\label{lem:pen-auxil}
The penalty term defined in Eq.~\eqref{eq:regsplit} is upper bounded by
\begin{align*}
    \Regpen \leq \E\left[ \frac{-\Psi( x_1)+\Psi( x^*)}{\eta_1} +\sum_{t=2}^T (\eta_{t}^{-1}-\eta_{t-1}^{-1})\left(-\Psi( x_t)+\Psi( x^*)\right)\right].
\end{align*}
\end{lemma}
\begin{proof}
We proceed as follows:
\begin{align*}
& \sum_{t=1}^T  \left(-\Phi_t(-\hat{L}_t) + \Phi_t(-\hat{L}_{t-1}) - \langle x^*, \hat{\ell}_t \rangle  \right)\\
&= \sum_{t=1}^T \left( \min_{x\in \conv(\mathcal{X})}\left\{  \langle x, \hat{L}_{t} \rangle + \eta_t^{-1}\Psi(x) \right\} - \left(  \langle x_t, \hat{L}_{t-1} \rangle  + \eta_t^{-1}\Psi(x_t) \right) \right\} -  \sum_{t=1}^T  \langle x^*, \hat{\ell}_t \rangle   \tag{by the definitions of $\Phi_t$ and $x_t$}\\
&\leq \langle x^*, \hat{L}_T \rangle + \eta_T^{-1}\Psi(x^*) + \sum_{t=1}^{T-1} \left(\langle x_{t+1}, \hat{L}_{t} \rangle + \eta_t^{-1}\Psi(x_{t+1})\right) - \sum_{t=1}^T \left(\langle x_t, \hat{L}_{t-1} \rangle + \eta_t^{-1}\Psi(x_t) \right)  - \langle x^*, \hat{L}_T \rangle \\
& = \eta_T^{-1}\Psi(x^*) + \sum_{t=2}^T \eta_{t-1}^{-1}\Psi(x_t) - \sum_{t=1}^T \eta_t^{-1}\Psi(x_t)   \tag{by telescoping and $\hat{L}_0=\mathbf{0}$} \\
& = \frac{-\Psi( x_1)+\Psi( x^*)}{\eta_1} +\sum_{t=2}^T (\eta_{t}^{-1}-\eta_{t-1}^{-1})\left(-\Psi( x_t)+\Psi( x^*)\right). 
\end{align*}
Finally using $\E\left[\ell_t \right] = \E[\hat{\ell}_t]$ and plugging in the definition of $\Regpen$ finish the proof. 
\end{proof}

\begin{proof}[Proof of Lemma~\ref{lem:pen}]
We directly plug into Lemma~\ref{lem:pen-auxil} the learning rate $\eta_t = \LearningRate$ and the regularizer $\Psi(x)=\Regularizer$.
Since $\gamma\leq 1$ and $-(1-x) \log(1-x)\leq \frac{\sqrt{x}}{2}$ for $x\in[0,1]$, we get
\begin{align*}
-\Psi( x_t)+\Psi( x^*) &= \sum_{i=1}^d \sqrt{x_{ti}}-\gamma(1-x_{ti})\log(1-x_{ti}) -\sum_{i:x^*_i=1} \sqrt{1} \nonumber\\
&\leq \sum_{i:x^*_i=0} \frac{3}{2}\sqrt{x_{ti}}-\sum_{i:x^*_i=1}\gamma(1-x_{ti})\log(1-x_{ti}) \\
&\leq \frac{3}{2}\left(\sum_{i:x^*_i=0} \sqrt{x_{ti}}-\sum_{i:x^*_i=1}\gamma(1-x_{ti})\log(1-x_{ti})\right).
\end{align*}
It further holds that $\eta_1=\eta_1^{-1}$ and 
\begin{align*}
\eta_t^{-1}-\eta_{t-1}^{-1} = \sqrt{t}-\sqrt{t-1}\leq \frac{1}{2\sqrt{t-1}}\leq \frac{1}{\sqrt{t}} = \eta_t.
\end{align*}
Inserting everything into Lemma~\ref{lem:pen-auxil}:
\begin{align*}
    \Regpen &\leq \E\left[ \frac{-\Psi( x_1)+\Psi( x^*)}{\eta_1} +\sum_{t=2}^T (\eta_{t}^{-1}-\eta_{t-1}^{-1})\left(-\Psi( x_t)+\Psi( x^*)\right)\right]\\
&\leq \E\left[\sum_{t=1}^T \eta_t \left(-\Psi( x_t)+\Psi( x^*)\right)\right]\\
& \leq \E\left[\sum_{t=1}^T \frac{3}{2\sqrt{t}}\left(\sum_{i:x^*_i=0} \sqrt{x_{ti}}-\sum_{i:x^*_i=1}\gamma(1-x_{ti})\log(1-x_{ti})\right)\right]\\
&\leq \sum_{t=1}^T \frac{3}{2\sqrt{t}}\left(\sum_{i:x^*_i=0} \sqrt{\E[x_{ti}]}-\sum_{i:x^*_i=1}\gamma(1-\E[x_{ti}])\log(1-\E[x_{ti}])\right).
\end{align*}
where the last step follows from Jensen's inequality and the concavity of functions $\sqrt{x}$ and $-(1-x)\log(1-x)$.
\end{proof}

\subsection{Stability term}
Bounding the stability term defined in Eq.~\eqref{eq:regsplit} requires tools from convex analysis.
First we extend the domain of $\Psi$ to $\mathbb{R}^d$ by setting $\Psi(x) = \infty, \;\forall x\in \mathbb{R}^d\setminus [0,1]^d$.
Recall the convex conjugate of a convex function $f$ is defined as
\begin{align*}
f^*(\cdot) = \max_{x\in\mathbb{R}^d} \scalar{x}{\cdot} - f(x),
\end{align*}
and the {\it Bregman divergence} associated with $f$ is defined as
\begin{align*}
D_{f} (x,y) = f(x)-f(y) -\scalar{\nabla f(y)}{x-y}.
\end{align*}
By the above definition, $\Phi_t$ can be written as $(\Psi_t + \mathcal{I}_{\conv}(\mathcal{X}))^{*}$. Note that $\Psi_t^*$ differs from $\Phi_t$ because it does not constrain its maximizer to be within $\conv(\X)$. 
The following properties hold (see, e.g., Chapter 7 of~\cite{bertsekas2003convex}): 
\begin{align}
    \nabla\Phi_t(\cdot) &= \argmax_{x\in\conv(\mathcal{X})}\scalar{x}{\cdot} - \Psi_t(x),   \label{eqn:legendre property 1} \\
    \nabla\Psi^*_t(\cdot) &= \argmax_{x\in[0,1]^d}\scalar{x}{\cdot} - \Psi_t(x).   \label{eqn:legendre property 2}
\end{align}
For $\Psi_t$ and $\Psi_t^*$, we have
\begin{align}
    &\nabla\Psi_t = (\nabla\Psi_t^*)^{-1}, \label{eqn:legendre property 3} \\    
    &\nabla^2 \Psi_t(x) = \left(\nabla^2\Psi_t^*(\nabla\Psi_t(x))\right)^{-1}. \label{eqn:legendre property 4}
\end{align} 
Furthermore, by Taylor's theorem, for any $x,y \in \mathbb{R}^d$ there exists a $z \in \conv(\{x,y\})$ such that
\begin{align}
    D_{\Psi^*_t}(x,y) = \frac{1}{2}||x-y||^2_{\nabla^2\Psi^*_t(z)}\label{eq:breg-2nd}.
\end{align}

The explicit expressions for $\nabla\Psi_t, \nabla^2 \Psi_t$ and a convenient upper bound for $(\nabla^2 \Psi_t)^{-1}$ in the domain $(0,1)^d$ are
\begin{align}
\Psi_t(x) &= \eta_t^{-1}\left(\Regularizer\right), \notag\\
\nabla\Psi_t(x) &= \eta_t^{-1} \left( -\frac{1}{2\sqrt{x_i}}-\gamma\log(1-x_i)-\gamma\right)_{i=1,\dots,d}, \notag\\
\nabla^2\Psi_t(x)  &= \eta_t^{-1} \diag\left[\left(\frac{1}{4\sqrt{x_i^3}}+\frac{\gamma}{1-x_i}\right)_{i=1,\dots,d}\right],\label{eq:hessian}\\
\left(\nabla^2\Psi_t(x)\right)^{-1} &\preceq \eta_t \diag\left[\left(\min\left\{4\sqrt{x_i^3},\gamma^{-1}(1-x_i)\right\}\right)_{i=1,\dots,d}\right],  \label{eq:inv-hessian}
\end{align}
where $(v_i)_{i=1,\ldots,d}$ denotes $(v_1, \ldots, v_d)$,
$\diag[(v_i)_{i=1,\ldots,d}]$ denotes a diagonal matrix with $(v_i)_{i=1,\ldots,d}$ on the diagonal,
and $A \preceq B$ for two matrices $A$ and $B$ means $B-A$ is positive semidefinite.
Note $\nabla\Psi_t$ is a bijection from $(0,1)^d$ to $\mathbb{R}^{d}$.
Therefore $\nabla\Psi^*_t(L)\in (0,1)^d$ for any $L\in\mathbb{R}^d$,
and all $x_t$'s we consider here are in the domain $(0,1)^d$.

The following Lemma will be useful to show that the stability term can be bounded independently of the action set $\X$.
\begin{lemma}
\label{lem:divergence_ineq}
For any $L$, let $\tilde L = \nabla\Psi_t(\nabla\Phi_t(L))$.
Then it holds for any $\ell\in\mathbb{R}^d$:
\begin{align*}
D_{\Phi_t}(L+\ell, L) \leq D_{\Psi_t^*}(\tilde{L}+\ell,\tilde{L}).
\end{align*}
\end{lemma}
\begin{proof}
First we state two equalities that follow from the previously stated properties. 
\begin{align}
    \nabla\Psi_t^*(\tilde{L}) &= \nabla\Psi_t^*(\nabla\Psi_t(\nabla\Phi_t(L))) \stackrel{\text{Eq.}\eqref{eqn:legendre property 3}}{=} \nabla\Phi_t(L),   \label{eq:eq1}\\
    \Psi_t^*(\tilde{L}) & \stackrel{\text{Eq.}\eqref{eqn:legendre property 2}}{=}  \scalar{\nabla\Psi_t^*(\tilde{L})}{\tilde{L}} - \Psi_t(\nabla\Psi_t^*(\tilde{L})) \nonumber \\
    &= \scalar{\nabla\Phi_t(L)}{\tilde{L}} - \Psi_t(\nabla\Phi_t(L)) \nonumber \\
    &\stackrel{\text{Eq.}\eqref{eqn:legendre property 1}}{=} \Phi_t(L)+\scalar{\nabla\Phi_t(L)}{\tilde{L}-L}. \label{eq:eq2}
\end{align}
We then proceed as follows:
\begin{align*}
    &D_{\Psi_t^*}(\tilde{L}+\ell,\tilde{L}) \\
    &= \Psi_t^*(\tilde L+\ell)-\Psi_t^*(\tilde{L})-\left\langle \nabla\Psi^*_t(\tilde L), \ell \right\rangle \tag{definition of Bregman divergence}\\ 
    &=\Psi_t^*(\tilde L+\ell)-\Phi_t(L)-\left\langle\nabla\Phi_t(L),\tilde L - L\right\rangle -\left\langle \nabla\Phi_t(L), \ell \right\rangle   \tag{by Eq.~~\eqref{eq:eq1} and \eqref{eq:eq2}}\\
    &=\Psi_t^*(\tilde L+\ell)-\Phi_t(L) -\left\langle \nabla\Phi_t(L), \tilde L - L+\ell \right\rangle  \\
    &\geq \scalar{ \nabla\Phi_t(L+\ell)}{ \tilde L +\ell} - \Psi_t(\nabla\Phi_t(L+\ell)) -\Phi_t(L) -\left\langle \nabla\Phi_t(L),\tilde L - L+ \ell \right\rangle  \tag{$\Psi_t^*$ is defined as the maximum} \\
    &= \scalar{ \nabla\Phi_t(L+\ell)}{ L +\ell}-\Psi_t(\nabla\Phi_t(L+\ell)) +\left\langle \nabla\Phi_t(L+\ell), \tilde L - L\right\rangle-\Phi_t(L)-\left\langle \nabla\Phi_t(L), \tilde L - L+\ell \right\rangle \\
    &= \Phi_t(L+\ell) +\left\langle \nabla\Phi_t(L+\ell), \tilde L - L\right\rangle-\Phi_t(L)-\left\langle \nabla\Phi_t(L), \tilde L - L+\ell \right\rangle \tag{by the definition of $\Phi_t$ and Eq.~\eqref{eqn:legendre property 1}}\\
    &= D_{\Phi_t}(L+\ell,L) +\left\langle \nabla\Phi_t(L+\ell)-\nabla\Phi_t(L),\tilde L - L\right\rangle\\
    &= D_{\Phi_t}(L+\ell,L) +\left\langle \nabla\Phi_t(L+\ell)-\nabla\Phi_t(L), \nabla\Psi_t(\nabla\Phi_t(L)) - L \right\rangle \\
    &\geq D_{\Phi_t}(L+\ell,L).
\end{align*}
The last step is by the first-order optimality condition: for the maximizer $\nabla\Phi_t(L):=\argmax_{x\in\conv(\X)}\scalar{x}{L}-\Psi_t(x)$ it must hold that 
$
\scalar{y - \nabla\Phi_t(L)}{ L - \nabla\Psi_t(\nabla\Phi_t(L))} \leq 0
$
for any $y\in\conv(\X)$.
\end{proof}

The next Lemma will be useful to bound the eigenvalues of the Hessian of $\Psi^*_t$.
\begin{lemma}
\label{lem:tilde_alpha_bound}
If $\eta_t \leq \min\{\frac{\sqrt{2}-1}{2},\frac{\gamma\log(2)}{4}\}$, then for any $x\in(0,1)^d$ and $\hat\ell$ such that $-1 \leq \hat{\ell}_i \leq \frac{2}{x_i}$ for all $i$, we have
\begin{align*}
2 x_{i}-1 \leq \nabla\Psi_t^*(\nabla\Psi_t(x)-\hat\ell)_i \leq 2 x_{i}.
\end{align*}
\end{lemma}
\begin{proof}
The functions $\nabla\Psi_t$ and $\nabla\Psi^*_t$ are symmetric and independent in each dimension. Therefore it is sufficient to consider $d=1$ and drop the index $i$.

For the upper bound we can assume $ x<\frac{1}{2}$; otherwise the statement is trivial since the range of $\nabla\Psi_t^*$ is $(0,1)^d$.
Now assume the opposite holds: $\nabla\Psi_t^*(\nabla\Psi_t(x)-\hat\ell) > 2 x$,
then we have
\begin{align*}
\hat\ell &= \nabla\Psi_t( x)-\nabla\Psi_t(x)+\hat\ell = \nabla\Psi_t( x)-\nabla\Psi_t(\nabla\Psi_t^*(\nabla\Psi_t(x)-\hat\ell)) \\ 
& < \nabla\Psi_t( x)-\nabla\Psi_t(2 x)   \tag{$\nabla\Psi_t(x)$ is strictly increasing in $(0,1)$} \\
&= \eta_t^{-1}\left(-\frac{1}{2\sqrt{ x}}-\gamma\log(1- x)+\frac{1}{2\sqrt{2 x}}+\gamma\log(1-2 x)\right)\\
&< -\eta_t^{-1}\left(\frac{\sqrt{2}-1}{2\sqrt{2}}\right)\frac{1}{\sqrt{ x}} \\
& < -\eta_t^{-1}\left(\frac{\sqrt{2}-1}{2}\right) \tag{ $x\leq \frac{1}{2}$} .
\end{align*}
The last line is a contradiction to the conditions $\eta_t\leq \frac{\sqrt{2}-1}{2}$ and $\hat\ell \geq -1$.

For the lower bound we can assume $ x > \frac{1}{2}$, otherwise the statement is again trivial.
Assume the opposite holds: $\nabla\Psi_t^*(\nabla\Psi_t(x)+\hat\ell) < 2 x -1$,
then we have
\begin{align*}
\hat\ell &= \nabla\Psi_t( x)-\nabla\Psi_t(\nabla\Psi_t^*(\nabla\Psi_t(x)-\hat\ell)) \\
&> \nabla\Psi_t( x)-\nabla\Psi_t(2 x-1) \tag{$\nabla\Psi_t(x)$ is strictly increasing in $(0,1)$} \\
&= \eta_t^{-1}\left(-\frac{1}{2\sqrt{ x}}-\gamma\log(1- x)+\frac{1}{2\sqrt{2 x-1}}+\gamma\log(2-2 x)\right)\\
&> \eta_t^{-1}\log(2) > \eta_t^{-1}\frac{\gamma\log(2)}{4}\frac{2}{ x} \tag{$\gamma \leq 1$ and $x > 1/2$}.
\end{align*}
which again leads to a contradiction to the conditions $\eta_t \leq\frac{\gamma\log(2)}{4}$ and $\hat\ell \leq \frac{2}{x}$.
This finishes the proof.
\end{proof}

Finally we are ready to prove Lemma~\ref{lem:stab}.
\begin{proof}[Proof of Lemma~\ref{lem:stab}]
Let $\tilde x_t = \nabla\Psi^*(\nabla\Psi( x_t)-\hat\ell_t)$. Define $\mathcal A_t= \bigotimes_{i=1}^d[ x_{ti}, \tilde x_{ti}]$. 
For any $t_0\geq 0$, we bound the stability term by
\begin{align}
    \Regstab &= \E\left[\sum_{t=1}^T\scalar{X_t}{\ell_{t}} + \Phi_t(-\hat L_t) - \Phi_t(-\hat L_{t-1})\right]\nonumber\\
&\stackrel{(1)}{\leq} \E\left[\sum_{t=t_0}^T\scalar{X_t}{\ell_{t}} + \Phi_t(-\hat L_t) - \Phi_t(-\hat L_{t-1})\right]+2t_0 m \nonumber\\
&\stackrel{(2)}{=} \E\left[\sum_{t=t_0}^T \E_t\left[\scalar{ x_t}{\hat{\ell}_t } + \Phi_t(-\hat L_t) - \Phi_t(-\hat L_{t-1})\right]\right]+2t_0 m \nonumber\\
&= \E\left[\sum_{t=t_0}^T \E_t\left[D_{\Phi_t}(-\hat L_t, -\hat L_{t-1})\right]\right]+t_0 m \nonumber\\
&\stackrel{(3)}{\leq}\E\left[\sum_{t=t_0}^T \E_t\left[D_{\Psi_t^*}(\nabla\Psi_t( x_t)-\hat\ell_t, \nabla\Psi_t( x_t))\right]\right]+2t_0 m \nonumber\\
&=\E\left[\sum_{t=t_0}^T \E_t\left[D_{\Psi_t^*}(\nabla\Psi_t( \tilde{x}_t), \nabla\Psi_t( x_t))\right]\right]+2t_0 m \nonumber\\
&\stackrel{(4)}{=} \E\left[\sum_{t=t_0}^T \E_t\left[\frac{1}{2}||\hat\ell_t||^2_{\nabla^2\Psi_t^*(z_t)}\right]\right]+2t_0 m \nonumber\\
&\stackrel{(5)}{\leq}\E\left[\sum_{t=t_0}^T \E_t\left[\max_{ x\in\mathcal A_t}\frac{1}{2}||\hat\ell_t||^2_{\nabla^2\Psi_t( x)^{-1}}\right]\right]+2t_0 m \nonumber\\
&=\E\left[\sum_{t=t_0}^T \E_t\left[\max_{ x\in\mathcal A_t}\frac{\eta_t}{2}||\hat\ell_t||^2_{\nabla^2\Psi( x)^{-1}}\right]\right]+2t_0 m. \label{eq:stab_as_norm}
\end{align}
(1) The difference of potentials for each step is bounded by $\Phi_t(-\hat L_t) - \Phi_t(-\hat L_{t-1})\leq \langle \nabla\Phi_{t}(-\hat L_{t}), -\hat\ell\rangle\leq ||\nabla\Phi_{t}(-\hat L_{t})||_1\leq m$, and the loss $\scalar{X_t}{\ell_t}$ is bounded by $m = \max_{x\in \X}||x||_1$.
(2) By the tower rule of conditional expectation, the unbiaseness of $\hat\ell$ and the sampling assumption, it holds that
\begin{align*}
\E\left[\scalar{X_t}{\ell_{t}}\right]=\E\left[\E_t\left[\scalar{X_t}{\ell_{t}}\right]\right] = \E\left[\E_t\left[\scalar{ x_t}{\ell_{t}}\right]\right] =\E\left[ \E_t\left[\scalar{ x_t}{\hat\ell_{t}}\right]\right].
\end{align*}

(3) Applyication of Lemma~\ref{lem:divergence_ineq}. 

(4) Property~\ref{eq:breg-2nd} ensures that some $z_t\in\conv(\{\nabla\Psi_t(x),\nabla\Psi_t(\tilde{x})\})$ exists that satisfies the equality. 

(5) By property~\eqref{eqn:legendre property 4} and the coordinate-wise monotonicity of $\nabla\Psi^*_t$ so that $\nabla\Psi^*_t\left(z_t \right) \subset \bigotimes_{i=1}^d[x_{ti},\tilde{x}_{ti}] = \mathcal A_t$.

We choose $t_0 = 58\gamma^{-2}$ such that $\eta_{t} \leq \min\{\frac{\sqrt{2}-1}{2},\frac{\gamma\log(2)}{4}\}$ for any $t \geq t_0$.
By the construction of $\hat{\ell}_t$ we clearly have $-1 \leq \hat{\ell}_{ti} \leq \frac{2}{x_{ti}}$.
We can then apply Lemma~\ref{lem:tilde_alpha_bound} 
to conclude that $\tilde{x}_{ti} \in [2x_{ti}-1, 2x_{ti}]$.
Therefore, with the form of Hessian \eqref{eq:inv-hessian} we have:
\begin{align*}
\forall x\in\mathcal A_t:\qquad  \nabla^2\Psi( x)^{-1} \leq \diag\left[\left(\min\left\{4\sqrt{(2x_{ti})^3},2\gamma^{-1}(1-x_{ti})\right\}\right)_{i=1,\dots,d}\right],
\end{align*}
and therefore,
\begin{align}
\sum_{t=t_0}^T \E_t\left[\max_{ x\in\mathcal A_t}\frac{\eta_t}{2}||\hat\ell_t||^2_{\nabla^2\Psi( x)^{-1}}\right] &\leq \sum_{t=t_0}^T \E_t\left[\frac{\eta_t}{2}\sum_{i=1}^d(\hat\ell_{ti})^2\min\{4\sqrt{(2x_{ti})^3},2\gamma^{-1}(1-x_{ti})\}\right]\nonumber   \\
&\stackrel{(1)}{\leq}\sum_{t=t_0}^T\frac{\eta_t}{2}\sum_{i=1}^d\frac{4}{x_{ti}}\min\{4\sqrt{(2x_{ti})^3},2\gamma^{-1}(1-x_{ti})\}\nonumber\\
&\stackrel{(2)}{\leq}\sum_{t=t_0}^T16\sqrt{2}\eta_t\sum_{i=1}^d\min\{\sqrt{x_{ti}},\gamma^{-1}(1-x_{ti})\}\nonumber\\
&\leq \sum_{t=1}^T\frac{16\sqrt{2}}{\sqrt{t}}\left(\sum_{i:x^*_i=0}\sqrt{x_{ti}}+\sum_{i:x^*_i=1}\gamma^{-1}(1-x_{ti})\right).\label{eq:norm_bound}
\end{align}
(1) Conditioned on $\mathcal{F}_{t-1}$, only $(\hat\ell_{ti})^2$ is random and its expectation is
\begin{align*}
\E_t\left[(\hat\ell_{ti})^2\right] = x_{ti}\left(\frac{\ell_{ti}+1}{x_{ti}}-1\right)^2 + 1-x_{ti} = \frac{4}{x_{ti}}\frac{(\ell_{ti}+1)^2-2(\ell_{ti}+1)x_{ti}+x_{ti}}{4}\leq 
 \frac{4}{x_{ti}}\frac{(4-4x_{ti}+x_{ti})}{4} \leq\frac{4}{x_{ti}}. 
\end{align*}

(2) Note that it always holds
\[
\frac{4}{x_{ti}}\min\left\{4\sqrt{(2x_{ti})^3},2\gamma^{-1}(1-x_{ti})\right\}
\leq \frac{16}{x_{ti}} \sqrt{(2x_{ti})^3} = 32\sqrt{2x_{ti}}.
\]
So it suffices to prove $\frac{4}{x_{ti}}\min\{4\sqrt{(2x_{ti})^3},2\gamma^{-1}(1-x_{ti})\} \leq 32\sqrt{2}\gamma^{-1}(1-x_{ti})$.
We consider two cases: 
(A) If $4\sqrt{(2x_{ti})^3} \leq 2\gamma^{-1}(1-x_{ti})$,
then we need to prove $\sqrt{x_{ti}} \leq \gamma^{-1}(1-x_{ti})$.
This is true since either $x_{ti} \geq 1/\sqrt{32}$ and thus $\sqrt{x_{ti}} \leq
2\sqrt{(2x_{ti})^3}  \leq \gamma^{-1}(1-x_{ti})$,
or $x_{ti} < 1/\sqrt{32}$ in which case $\sqrt{x_{ti}} \leq 1- x_{ti} \leq \gamma^{-1}(1-x_{ti})$.
(B) If $4\sqrt{(2x_{ti})^3} \geq 2\gamma^{-1}(1-x_{ti})$,
then $x_{ti}$ must be larger than $1/4$. 
In this case we bound $\frac{1}{x_{ti}}$ by 4 and the desired inequality follows.

The proof is concluded by inserting Eq.~\eqref{eq:norm_bound} into Eq.~\eqref{eq:stab_as_norm} and using Jensen's inequality to move the expectation into the concave functions.
\end{proof}

\subsection{General upper bounds for $C_{sto}, C_{add}$ and $C_{adv}$}

We now finish the proof of Theorem~\ref{th:arbitrary} on the upper bounds of the three constants.
\ \\

\textbf{Bounding $\mathbf{C_{adv}}$:}
\begin{align*}
    C_{adv} 
   &= \max_{x\in\conv(\X)} \sum_{i: x_i^*=0}\sqrt{x_i} +  \sum_{i: x_i^*=1} (\gamma^{-1}-\gamma\log(1-x_i))(1-x_i) \\
   &\leq \max_{x\in\conv(\X)}  \sum_{i: x_i^*=0}\sqrt{x_i} +  \sum_{i: x_i^*=1} \gamma\sqrt{1-x_i} + \sum_{i: x_i^*=1} \gamma^{-1}(1-x_i)  \tag{$-y\log y \leq \sqrt{y}$ for $y\in[0,1]$}  \\
   &\leq  \max_{x\in\conv(\X)} \sqrt{\left(\sum_{i: x_i^*=0}1\right)\left(\sum_{i: x_i^*=0} x_i\right)} +\gamma \sqrt{\left(\sum_{i: x_i^*=1}1\right)\left(\sum_{i: x_i^*=1}(1-x_i)\right)} + \gamma^{-1}m   \tag{Cauchy-Schwarz} \\
   &\leq \sqrt{dm} + \gamma m + \gamma^{-1}m \\
   &\leq \mathcal{O}\left( \gamma^{-1}\sqrt{md} \right). 
\end{align*}

\textbf{Bounding $\mathbf{C_{sto}}$:}  $C_{sto}$ is defined as $\max_{\alpha\in [0,\infty)^{|\X|}} f(\overline{\alpha}) - r(\alpha)$. First we bound $f(\overline{\alpha})$: 
\begin{align*}
     f(\overline{\alpha}) = \sum_{i: x_i^*=0} \sqrt{\overline{\alpha}_i} = \sum_{i: x_i^*=0} \sqrt{\sum_{x\in\X}\alpha_x x_i} = \sum_{i: x_i^*=0} \sqrt{\sum_{x\in\X\backslash \{x^*\}}\alpha_x x_i} \leq \sqrt{d\sum_{i:x_i^*=0}\sum_{x\in\X\backslash \{x^*\}}\alpha_x x_i}
\leq \sqrt{dm\sum_{x\in\X\backslash \{x^*\}}\alpha_x}. 
\end{align*}
On the other hand, 
\begin{align*}
     r(\alpha) = \sum_{x\in \X\backslash \{x^*\}} \alpha_x \Delta_{x} \geq \Delta_{\min} \sum_{x\in \X\backslash \{x^*\}} \alpha_x. 
\end{align*}
Combining them we get 
\begin{align*}
    C_{sto}&\leq \max_{\alpha\in[0,\infty)} \sqrt{dm \sum_{x\in\X\backslash\{x^*\}}\alpha_x}- \Delta_{\min}\sum_{x\in\X\backslash\{x^*\}}\alpha_x \\
                &\leq \max_{A\geq 0} \sqrt{dm A} - \Delta_{\min}A \\
                &\leq \max_{A\geq 0}  \Delta_{\min}A + \frac{dm}{4\Delta_{\min}}  - \Delta_{\min}A \tag{AM-GM inequality} \\
                &=  \frac{dm}{4\Delta_{\min}}. 
\end{align*}

\textbf{Bounding $\mathbf{C_{add}}$:} 
Recall $C_{add}$ is defined as $\sum_{t=1}^{\infty} \max_{\alpha \in \Delta(\X)} \left( \frac{100}{\sqrt{t}}g(\overline \alpha)-r(\alpha) \right)$.
We will give a upper bound for $g(\overline \alpha)$ and lower bound for $r(\alpha)$ below.

We first prove the following property: for any $y\in \mathbb{R}_+^N$, $\sum_{i=1}^N y_i\log\frac{1}{y_i} \leq \|y\|_1 \log\frac{N}{\|y\|_1}$. Indeed, by the concavity of the $\log$ function and Jensen's inequality, 
\begin{align*}
     \sum_{i=1}^N \frac{y_i}{\|y\|_1}\log\frac{1}{y_i} \leq \log\left( \sum_{i=1}^N \frac{y_i}{\|y\|_1} \frac{1}{y_i} \right) = \log \frac{N}{\|y\|_1}. 
\end{align*}

Therefore, for any $\alpha\in \Delta(\mathcal{X})$ we have
\begin{align*}
     g(\overline{\alpha}) &= \sum_{i: x_i^*=1}\left(\gamma^{-1} + \gamma\log\left(\frac{1}{1-\bar{\alpha}_i}\right)\right)(1-\bar{\alpha}_i) \\
     &\leq \left(\sum_{i: x_i^*=1}(1-\overline{\alpha}_i) \right)\left(\gamma^{-1} + \gamma \log \frac{m}{\sum_{i: x_i^*=1} (1-\overline{\alpha}_i)}\right)  \tag{using the above property}.\\
\end{align*}
Then consider the following two facts. First, the function of $y$ defined by $y(\gamma^{-1} + \gamma \log \frac{m}{y} )$ is increasing in $y\in [0, m]$. This can be verified by
\begin{align*}
\frac{\partial}{\partial y}\left(y\left(\gamma^{-1} + \gamma \log \frac{m}{y} \right)\right) = \gamma^{-1}+\gamma \log m - \gamma\log y - \gamma \geq 0. \tag{$\gamma\leq 1$}
\end{align*}
Second, we have $\sum_{i:x_i^*=1}(1-\overline{\alpha}_i) = \sum_{i:x_i^*=1}\sum_{\alpha\in \X} \alpha_x(1-x_i) = \sum_{i:x_i^*=1}\sum_{\alpha\in \X\backslash\{x^*\} }\alpha_x(1-x_i) \leq \|x^*\|_1 \left(\sum_{\alpha\in \X\backslash\{x^*\} }\alpha_x\right) \leq m \left(\sum_{\alpha\in \X\backslash\{x^*\} }\alpha_x\right)$. Combining these two facts with the above bound for $g(\overline{\alpha})$, we get 
\begin{align*}
    g(\overline{\alpha})  \leq m \left(\sum_{\alpha\in \X\backslash\{x^*\} }\alpha_x\right)\left(\gamma^{-1} + \gamma \log \frac{1}{\sum_{\alpha\in\X\backslash \{x^*\}}\alpha_x}\right).   
\end{align*}
On the other hand, we have the lower bound for $r(\alpha)$: 
\begin{align*}
    r(\alpha) = \sum_{x\in \mathcal{X}\backslash \{x^*\}} \alpha_x \Delta_x \geq  \Delta_{\min} \sum_{x\in \mathcal{X}\backslash \{x^*\}} \alpha_x.   
\end{align*}
Therefore, 
\begin{align*}
C_{add}&= \sum_{t=1}^{\infty} \max_{\alpha\in \Delta(|\X|)} \left( \frac{100}{\sqrt{t}}g(\overline \alpha) - r(\alpha) \right) \\
&\leq \sum_{t=1}^{\infty} \max_{A\in [0,1]} \left(\frac{100}{\sqrt{t}} mA \left(\gamma^{-1}+\gamma\log \frac{1}{A}\right) - \Delta_{\min} A\right). 
\end{align*}
We further bound it by the sum of the following two summations: 
\begin{itemize}
\item $\displaystyle \sum_{t=1}^{\infty} \max_{A\in [0,1]} \left(\frac{100}{\sqrt{t}} mA \gamma^{-1}  - \frac{1}{2}\Delta_{\min} A\right)$
\item $\displaystyle \sum_{t=1}^{\infty} \max_{A\in [0,1]} \left( \frac{100}{\sqrt{t}}m A\gamma \log \frac{1}{A} - \frac{1}{2} \Delta_{\min} A\right)$
\end{itemize} 
Lemma~\ref{lemma: another tedius lemma} and~\ref{lemma: tedius lemma} below respectively bound these two as $\mathcal{O}\left( \frac{m^2\gamma^{-2}}{\Delta_{\min}} \right)$ and $\mathcal{O}\left( \frac{m^2 \gamma^2}{\Delta_{\min}} \right)$, which finishes the proof.

\begin{lemma}
\label{lemma: another tedius lemma}
For any $C > 0$ and $\Delta > 0$, we have
$\sum_{t=1}^{\infty}\max_{A\in [0,1]} \left( \frac{C}{\sqrt{t}}A - \Delta A \right) \leq \mathcal{O}\left(\frac{C^2}{\Delta}\right)$. 
\end{lemma}
\begin{proof}
Let $T_0$ be the largest $t$ such that $\frac{C}{\sqrt{t}}-\Delta >0$, then 
\begin{align*}
\sum_{t=1}^{\infty}\max_{A\in [0,1]} \left( \frac{C}{\sqrt{t}}A - \Delta A \right) \leq  \sum_{t=1}^{T_0} \frac{C}{\sqrt{t}} \leq 2C\sqrt{T_0} = \mathcal{O}\left(\frac{C^2}{\Delta}\right). 
\end{align*}
\end{proof}

\begin{lemma}
\label{lemma: tedius lemma}
For any $C > 0$ and $\Delta > 0$, we have
$\sum_{t=1}^{\infty}\max_{A\in [0,1]} \left( \frac{C}{\sqrt{t}}A \log \frac{1}{A} - \Delta A \right) \leq \frac{C^2}{\Delta}$. 
\end{lemma}
\begin{proof}
We first solve the inner optimization with respect to a specific $t$. Taking the derivative with respect to $A$, and setting it to zero: 
\begin{align}
 \frac{C}{\sqrt{t}}\log \frac{1}{A^*}- \frac{C}{\sqrt{t}} - \Delta=0,  \label{eqn: intermediate}
\end{align}
we get the solution
\begin{align*}
A^* = \exp\left( -1-\frac{\sqrt{t}\Delta}{C} \right). 
\end{align*}
And thus, 
\begin{align*}
\max_{A\in [0,1]} \left(\frac{C}{\sqrt{t}}A\log \frac{1}{A} - \Delta A \right)
&= \frac{C}{\sqrt{t}}A^*\log \frac{1}{A^*} - \Delta A^*  
\stackrel{\text{Eq.}~\eqref{eqn: intermediate}}{=}A^*\left(\frac{C}{\sqrt{t}}+\Delta \right) - \Delta A^* \\
&=\frac{C}{\sqrt{t}}  \exp\left( -1-\frac{\sqrt{t}\Delta}{C} \right)
\end{align*}

Finally we have
\begin{align*}
\sum_{t=1}^\infty \max_{A\in [0,1]} \left(\frac{C}{\sqrt{t}}A\log \frac{1}{A} - \Delta A \right) 
&\leq \sum_{t=1}^{\infty} \frac{C}{\sqrt{t}} \exp\left( -1-\frac{\sqrt{t}\Delta}{C} \right)  
\leq  \int_{t=0}^\infty \frac{C}{\sqrt{t}}\exp\left( -1-\frac{\sqrt{t}\Delta}{C} \right) dt \\
&=   \frac{C^2}{\Delta}\int_{\tau=0}^\infty \frac{1}{\sqrt{\tau}} \exp(-1-\sqrt{\tau}) d\tau 
\leq \frac{C^2}{\Delta}. 
\end{align*}
\end{proof}


\section{Omitted Details for Sections~\ref{sec:hypercube} and~\ref{sec:m-set}}
\label{app:theorems}

In this section we provide omitted details for the two special cases:
full combinatorial set and $m$-set.

\subsection{Optimality of the stochastic bound when $\X=\{0,1\}^d$}
As mentioned in the proof of Theorem~\ref{th:full},
we provide here for completeness a proof showing that when $d=1$ and $\Delta > 0$,
the regret is at least $\Omega(\frac{\log T}{\Delta})$.

Assume that there exists an algorithm that is at least as good as ours asymptotically, which implies $\lim_{T\rightarrow\infty} \frac{\log(\Reg_T)}{\log(T)} \leq \lim_{T\rightarrow\infty} \frac{\log(\mathcal{O}\left(\log(T)\right)}{\log(T)} =0 $ for any problem.
For some $\Delta>0$ we consider two problems: $\E[\ell] = \Delta$ and $\E[\ell] = -\Delta$.
For simplicity we assume that the losses are drawn from i.i.d. Gaussian with variance $\sigma^2=1$, but the proof can be easily transferred to Bernoulli noise as well.
For the problem with positive loss, we denote the regret as $\Reg^+_T$ and the probability space induced by an algorithm by $\mathbb{P}^+$.
Equivalently we define $\Reg^-_T$ and $\mathbb{P}^-$.
The relative entropy between $\mathbb{P}^+$ and $\mathbb{P}^-$ is 
\begin{align*}
\text{\rm KL}(\mathbb{P}^+,\mathbb{P}^-) = \sum_{t=1}^T \mathbb{P}^+(X_t=1)(2\Delta)^2 = 4\Reg^+_T\Delta. 
\end{align*}
Also we have by the definition of regret:
\begin{align*}
\mathbb{P}^+\left(\sum_{t=1}^TX_t \geq \frac{T}{2}\right) + \mathbb{P}^-\left(\sum_{t=1}^TX_t < \frac{T}{2}\right) \leq \frac{2(\Reg_T^++\Reg_T^-)}{\Delta T}.
\end{align*}
Using the high probability Pinsker inequality (included after the proof for completeness), we get
\[
\frac{2(\Reg_T^++\Reg_T^-)}{\Delta T} \geq \frac{1}{2}\exp\left(-4\Reg^+_T\Delta\right).
\]
Rearranging gives
\[
\frac{\Reg^+_T}{\log(T)} 
= \frac{1}{4\Delta} - \frac{1}{4\Delta}\frac{\log(4(\Reg_T^++\Reg_T^-))}{\log(T)} + \frac{\log (\Delta)}{4\Delta\log(T)}. 
\]
Taking the limit on both sides shows $\lim_{T\rightarrow\infty} \frac{\Reg^+_T}{\log(T)} = \Omega(\frac{1}{\Delta})$,
which finishes the proof. \\

\begin{lemma}[High Probability Pinsker, e.g.~\citep{bubeck2013bounded}]
Let $\mathbb{P}$ and $\mathbb{Q}$ be probability measures on the same measurable
space $(\Omega, F)$ and let $A \in F$ be an arbitrary event. Then,
$$\mathbb{P}(A) + \mathbb{Q}(A^c) \geq \frac{1}{2}\exp(-\text{\rm KL}(\mathbb{P},\mathbb{Q})),$$
where $A^c$ is the complement of $A$ and $\text{\rm KL}(\mathbb{P},\mathbb{Q})$ the relative entropy.
\end{lemma}

\subsection{Sampling rule for $m$-set}
In this section $\X$ represents the $m$-set.
We first define the following auxiliary vectors for $0\leq i \leq m$, $0\leq j\leq d-m$.
\begin{align*}
\beta_{i,j} = \left(\underbrace{1,\dots,1}_{i},\frac{m-i}{d-i-j},\dots,\frac{m-i}{d-i-j},\underbrace{0,\dots,0}_{j}\right) \in \conv(\X).
\end{align*}
It is trivial to sample with mean $\beta_{i,j}$ with the sampling rule:
\begin{align*}
P_{i,j} = \text{Uniform}\left(\left\{x\in\X \;|\; x_{1,\dots,i}=\mathbf{1}\, \land \,x_{d-j+1,\dots,d}=\mathbf{0}\right\}\right).
\end{align*}
This requires uniform sampling of a $(m-i)$-sized subset of $(d-i-j)$ elements, which can be done in $\mathcal{O}(d)$ time.

Now for a given $x_t \in \conv(\X)$, one sampling rule $P$ such that $\E_{X\sim P}[X] = x_t$ is the following:
First we sort the entries of $x_t$ so that $x$ is the sorted version with $x_1 \geq \cdots \geq x_d$. This takes $\mathcal{O}(d\log(d))$ time.
Next we decompose $x = \sum_{s=0}^{d} p_{x,s} \beta_{i_s,j_s}$ such that $p_{x, s} \in [0,1]$, $\sum_{s=0}^{d}p_{x,s}=1$, $(i_0,j_0)=(0,0)$ and $(i_{s+1},j_{s+1})-(i_s,j_s)\in \{(1,0), (0,1)\}$. In other words, either $i$ or $j$ increases by one from $s$ to $s+1$. 
This decomposition is unique and can be computed in a greedy manner in time $\mathcal{O}(d)$.
Finally the full sampling scheme is $\sum_{s=0}^d p_{x,s}P_{i_s,j_s}$ (in terms of permuted coordinates).
The runtime is dominated by the sorting and hence is $\mathcal{O}(d\log(d))$ overall. 


\subsection{Complete proof for Theorem~\ref{th:m-set}}

\textbf{Bounding $\mathbf{C_{adv}}$:}
\begin{align*}
C_{adv} &=  \max_{x\in\conv(\X)}\left(f(x) + g(x)\right) = \max_{x\in\conv(\X)} \sum_{i:x^*_i=0}\sqrt{x_i}+\sum_{i:x^*_i=1}(\gamma^{-1}-\gamma\log(1-x_i))(1-x_i). 
\end{align*}

The optimization problem is concave in $x$ and symmetric for all $i$ with the same value of $x^*_i$. This implies that the $\argmax$ solution must take the following form: 
\begin{align*}
\left(\argmax_{x\in\conv(\X)}f(x) + g(x)\right)_i =  \begin{cases}\lambda &\mbox{ if }x^*_i = 0\\1-\frac{d-m}{m}\lambda &\mbox{ if }x^*_i = 1\end{cases}
\end{align*}
for some $\lambda\in [0,\min\{1,\frac{m}{d-m}\}]$. 

Therefore, 
\begin{align}
C_{adv} &= \max_{\lambda \in [0,\min(1,\frac{m}{d-m})]} (d-m)\sqrt{\lambda}+m\left(\gamma^{-1}-\gamma\log\left(\frac{d-m}{m}\lambda\right)\right)\frac{d-m}{m}\lambda \nonumber \\
&= \max_{\lambda \in [0,\min(1,\frac{m}{d-m})]} (d-m)\left(\sqrt{\lambda}+\left(\gamma^{-1}-\gamma\log\left(\frac{d-m}{m}\lambda\right)\right)\lambda\right).   \label{eq:optimization of lambda}
\end{align}
Since $\frac{d-m}{m}\lambda \leq 1$ and $\gamma\leq 1$, the derivative is always positive:
\begin{align*}
&\frac{\partial}{\partial \lambda}\left(\sqrt{\lambda}+\left(\gamma^{-1}-\gamma\log\left(\frac{d-m}{m}\lambda\right)\right)\lambda\right)\\
&=\left(\frac{1}{2\sqrt{\lambda}}+\gamma^{-1}-\gamma\log\left(\frac{d-m}{m}\lambda\right)-\gamma\right)\geq \frac{1}{2\sqrt{\lambda}} >0. 
\end{align*}
Therefore we can simply plug in the upper border of $\lambda$ in Eq.\eqref{eq:optimization of lambda}: 

Case $m \leq d/2$ (for which $\gamma=1$ and the optimal $\lambda$ is $m/(d-m)$): 
\begin{align*}
&C_{adv} = (d-m)\left(\sqrt{\frac{m}{d-m}}+\frac{m}{d-m}\right)\leq 2\sqrt{(d-m)m} = \mathcal{O}\left(\sqrt{md}\right). 
\end{align*}

Case $m > d/2$ (for which $\gamma=\min\left\{1, 1/\sqrt{\log\left(\frac{d}{d-m}\right)}\right\}$ and the optimal $\lambda$ is 1): \\
Note that $\gamma \leq \frac{1}{\sqrt{\log\left(\frac{d}{d-m}\right)}}$ and thus $\gamma^{-1} = \max\left\{1, \sqrt{\log\left(\frac{d}{d-m}\right)}\right\} \leq \frac{\sqrt{\log\left(\frac{d}{d-m}\right)}}{\sqrt{\log(2)}}$
and $-\gamma \leq -\frac{\sqrt{\log(2)}}{\sqrt{\log\left(\frac{d}{d-m}\right)}}$. Therefore
\begin{align*}
&C_{adv} \leq (d-m)\left(1+\frac{1}{\sqrt{\log(2)}}\sqrt{\log\left(\frac{d}{d-m}\right)} + \frac{\sqrt{\log(2)}}{\sqrt{\log\left(\frac{d}{d-m}\right)}} \log\left(\frac{m}{d-m}\right)\right) \\
&\leq (d-m)\left(1+\left(\frac{1}{\sqrt{\log(2)}} + \sqrt{\log(2)}\right)\sqrt{\log\left(\frac{d}{d-m}\right)}\right) 
= \mathcal{O}\left((d-m)\sqrt{\log\left(\frac{d}{d-m}\right)} \right).
\end{align*}

\textbf{Bounding $\mathbf{C_{sto}}$:}
With our definitions of $\Delta_i$, for any $x\in\X$, we have 
\begin{align}
\Delta_x = \E\left[ \sum_i (x_i-x_i^*)\ell_{ti}\right] = \E\left[\sum_i (x_i-x_i^*)(\ell_{ti}-\ell_{tm}) \right]
= \sum_{i: x_i^*=1}(1-x_i)|\Delta_i| + \sum_{i: x_i^*=0}x_i\Delta_i \geq \sum_{i: x^*_i=0}x_i\Delta_i, 
\end{align}
and thus for any $\alpha\in [0,\infty)^{|\X|}$
\begin{align}
r(\alpha) = \sum_{x\in\X\setminus\{x^*\}} \alpha_x\Delta_x \geq \sum_{x\in \X\backslash\{x^*\}}\sum_{i: x_i^*=0}  \alpha_x x_i\Delta_i = \sum_{i: x_i^*=0}\overline{\alpha}_i \Delta_i.  \label{eq:lower bound of r}
\end{align}
Therefore, 
\begin{align*}
C_{sto} &= \max_{\alpha\in[0,\infty)^{|\mathcal{X}|}}\sum_{i:x^*_i=0}\sqrt{\overline\alpha_i} -r(\alpha)\\
&\leq \max_{\overline\alpha\in[0,\infty)^d}\sum_{i:x^*_i=0}\left(\sqrt{\overline\alpha_i}-\overline\alpha_i\Delta_i\right) \\
&\stackrel{\text{AM-GM}}{\leq} \max_{\overline\alpha\in[0,\infty)^d}\sum_{i:x^*_i=0}\left(\overline\alpha_i \Delta_i + \frac{1}{4\Delta_i} -\overline\alpha_i\Delta_i\right) = \sum_{i:x^*_i = 0} \frac{1}{4\Delta_i}.
\end{align*}

\textbf{Bounding $\mathbf{C_{add}}$:}
Similar to the ``Bounding $C_{add}$'' part in the proof of Theorem~\ref{th:arbitrary} (earlier in Appendix~\ref{app:arbitrary}), we can bound for any $\alpha\in \Delta(\X)$: 
\begin{align*}
g(\overline\alpha)&=\sum_{i:x^*_i=1} \left(\gamma^{-1}+\gamma\log\left(\frac{1}{1-\overline\alpha_i}\right)\right)(1-\overline\alpha_i) \\
&\leq \left(\gamma^{-1}+\gamma\log\left(\frac{m}{\sum_{i:x^*_i=1}(1-\overline\alpha_i)}\right)\right)\sum_{i:x^*_i=1}(1-\overline\alpha_i)  \tag{by the concavity of $g$} \\
&=\left(\gamma^{-1}+\gamma\log\left(\frac{m}{\sum_{i:x^*_i=0}\overline\alpha_i}\right)\right)\sum_{i:x^*_i=0}\overline\alpha_i \\
&\leq \sum_{i:x^*_i=0}\left(\gamma^{-1}+\gamma\log\left(\frac{m}{\overline\alpha_i}\right)\right)\overline\alpha_i.
\end{align*}
where in the second equality we use an property of $m$-set: $\sum_{i:x^*_i=1} (1-\overline\alpha_i) = \sum_{i:x^*_i=0} \overline\alpha_i$, which follows from the fact that $\overline\alpha$ is in the convex hull of $m$-set. In the last inequality, we simply lower bound $\sum_{i: x_i^*=0 }\overline{\alpha}_i$
by one of its summands. 

Using the same lower bound 
\begin{align*}
r(\alpha)  \geq \sum_{i:x^*_i=0} \Delta_i\overline\alpha_i,    \tag{by Eq. \eqref{eq:lower bound of r}}
\end{align*}
we have an upper bound for $C_{add}$: 
\begin{align*}
C_{add} &= \sum_{t=1}^\infty \max_{\alpha\in \Delta(\X)} \frac{100}{\sqrt{t}}g(\overline\alpha) - r(\alpha) \\
&\leq\sum_{i:x^*_i=0}  \sum_{t=1}^{\infty} \max_{\overline\alpha_i \in [0,1]} \frac{100}{\sqrt{t}}\left(\gamma^{-1}+\gamma\log\frac{m}{\overline{\alpha}_i  } \right)\overline{\alpha}_i  - \Delta_{i}\overline{\alpha}_i \nonumber\\
&\leq \sum_{i:x^*_i=0} \left(
\sum_{t=1}^{\infty} \max_{\overline{\alpha}_i \in[0,1]} \left(\frac{100}{\sqrt{t}}\left(\gamma^{-1}+\gamma\log m\right)\overline{\alpha}_i -\frac{\Delta_{i}}{2}\overline{\alpha}_i\right) 
+ \sum_{t=1}^{\infty} \max_{\overline{\alpha}_i  \in[0,1]} \left(\frac{100}{\sqrt{t}}\gamma \overline{\alpha}_i\log\frac{1}{\overline{\alpha}_i} -\frac{\Delta_{i}}{2}\overline{\alpha}_i\right)
\right). 
\end{align*}
Invoking Lemma~\ref{lemma: another tedius lemma} and~\ref{lemma: tedius lemma} on the above two terms, we get 
\begin{align*}
C_{add}\leq \mathcal{O}\left(  \sum_{i: x_i^*=0}\frac{(\gamma^{-1} + \gamma \log m)^2}{\Delta_{i}}   \right). 
\end{align*}
This can be further upper bounded by $\mathcal{O}\left(\sum_{i: x_i^*=0}\frac{(\log d)^2}{\Delta_{i}}\right)$ by our selection of $\gamma$ in either regime.

\end{document}